\documentclass[oneside, 11pt]{article}
\usepackage[T1]{fontenc}
\usepackage{polyglossia}
\setmainlanguage{english}
\usepackage{ctex}

\usepackage{fancyhdr}

\fancyhfoffset[L,R]{0pt} 
\usepackage{mfirstuc} 

\usepackage{jmlr2e_cus}
\usepackage[usenames]{xcolor}
\usepackage[bookmarksnumbered=true]{hyperref}
\usepackage{url}
\definecolor{cite_color}{HTML}{114083}
\definecolor{link_color}{RGB}{153, 0,0}  
\definecolor{url_color}{RGB}{153, 102,  0}
\definecolor{emp_color}{RGB}{0,0,255}
\hypersetup{
 colorlinks,
 citecolor=cite_color,
 linkcolor=link_color,
 urlcolor=url_color}
\graphicspath{{./}}

\usepackage{framed}
\definecolor{shadecolor}{rgb}{0.94, 0.97, 1.0}

\usepackage{epsfig}
\usepackage{amsmath}
\usepackage{amssymb}
\usepackage{mathrsfs}
\usepackage{graphicx}
\usepackage{subfig}
\usepackage{multirow}
\usepackage{booktabs}
\usepackage{wrapfig}
\usepackage{enumerate}
\usepackage{bm}
\usepackage{tabularx}

\usepackage[backend=biber,natbib=true,style=numeric,sorting=nyt,maxcitenames=4,maxbibnames=199,language=english]{biblatex}

\usepackage{mathtools}

\usepackage[vlined, ruled, linesnumbered]{algorithm2e}

\SetCommentSty{mycommfont}
\SetKwInOut{Input}{input}
\SetKwInOut{Output}{output}

 \usepackage{cleveref} 
 \crefname{section}{Section}{Sections}
 \crefname{theorem}{Theorem}{Theorems}
 \crefname{lemma}{Lemma}{Lemmas}
 \crefname{equation}{Equation}{Equations}
 \crefname{proposition}{Proposition}{Propositions}
 \crefname{claim}{Claim}{Claims}
\crefname{appendix}{Appendix}{Appendices}
   \crefname{algorithm}{Algorithm}{Algorithms}
 \crefname{figure}{Figure}{Figures}
 \crefname{table}{Table}{Tables}
 \crefname{remark}{Remark}{Remarks}
 \crefname{definition}{Definition}{Definitions}
 \crefname{equatinon}{Equation}{Equations}
 \crefname{corollary}{Corollary}{Corollaries}

\allowdisplaybreaks

\usepackage{thmtools}
\usepackage{thm-restate}

\let \oldtextcircled \textcircled
\renewcommand{\textcircled}[1]{\oldtextcircled{\footnotesize #1}}

\usepackage{enumitem}
\setlist[itemize]{leftmargin=9mm}

\newtheorem{lemma}{Lemma}[section]

\firstpageno{1}

\usepackage{setspace}
\usepackage[table]{xcolor}
\definecolor{darkgreen}{rgb}{0.0,0.5,0.0}
\definecolor{darkred}{rgb}{0.5,0.0,0.0}
\usepackage{wrapfig}



\usepackage{float}       
\usepackage{algorithm}   
\usepackage{algorithmic} 

\usepackage{hyperref}

\usepackage{etoc}
\etocdepthtag.toc{mtchapter}
\etocsettagdepth{mtchapter}{subsection}
\etocsettagdepth{mtappendix}{none}

\usepackage{caption} 

\begin{document}

\title{\LARGE  
Unified Energy for Invariant and Independent Decoding in Diffusion Language Models
}

\author{
 \normalfont{Yuchen Yan} \\
 Department of Computer Science\\
 National University of Singapore\\
 \texttt{yan1998@comp.nus.edu.sg}\\
\And
 \normalfont{Minkai Xu} \\
 Department of Computer Science \\
 Stanford University\\
 \texttt{minkai@cs.stanford.edu}\\
\AND
\normalfont{Zaiquan Yang} \\
Department of Computer Science\\
City University of Hong Kong\\
\texttt{zaiquyang2-c@my.cityu.edu.hk}\\
\And
\normalfont{Yatao Bian}\thanks{Correspondence to: Yatao Bian <\texttt{ybian@nus.edu.sg}>.} \\
 Department of Computer Science\\
 National University of Singapore\\
 \texttt{ybian@nus.edu.sg}\\
\AND
\normalfont{\today}
}

\maketitle

\begin{abstract}
Diffusion Language Models (DLMs) enable parallel text generation by iteratively denoising a full sequence, offering attractive flexibility compared to auto-regressive (AR) decoding. However, existing methods fail to fully capture token relationships, leading to a performance gap relative to AR baselines, especially as the degree of parallelism increases. In this paper, we give a systematic analysis of the gap, identifying three key factors: (i) model capacity, (ii) dependency, and (iii) invariance. To address these issues, we first propose an invariant energy (Inv-E) together with an effective sampling-based estimator to handle the invariance issue. By further combining with the independent energy (Ind-E), we obtain a unified energy (Uni-E), that accounts for all these factors. Uni-E enjoys a unique advantage: it can be computed exactly without sampling-based partition estimation. Besides, Uni-E is model agnostic and can therefore be scaled to models of arbitrary size. We further prove that Uni-E can correct the distribution shift caused by dependency and invariance. Extensive experiments across Diffusion Language Models (DLMs) and Diffusion Large Language Models (DLLMs) demonstrate the effectiveness of the proposed Uni-E.
\end{abstract}

\newpage 
{
\hypersetup{linkcolor=blue}
\tableofcontents
}
\clearpage

\section{Introduction}
\par Diffusion Language Models (DLMs)~\cite{li2025survey,yi2024diffusion} have emerged as a compelling alternative to the dominant auto-regressive (AR) paradigm in text generation, where tokens are generated in a fixed left-to-right order~\cite{zhao2023survey,chang2024survey,zou2023survey}. Unlike AR-based methods, DLMs iteratively decode tokens across the entire sequence in parallel, leading to higher flexibility, diversity, and efficiency. However, traditional DLMs~\cite{sahoo2024simple,austin2021structured,lou2023discrete,arriola2025block,sahoo2025diffusion} still underperform AR-based generation when decoding in parallel, as they assume tokens can be decoded independently and invariantly, failing to model their relationships. To handle this limitation, one line of research uses score guidance~\cite{ye2025dream,wu2025fast,li2026breaking,luo2026dawn} or proxy models~\cite{liu2024discrete,xu2024energy} to capture token relationships. Another line of research~\cite{wang2025remasking,yao2026remask,zhai2026core} proposes correcting decoding errors in later steps through remasking strategies. Nevertheless, these approaches either lack a systematic understanding of DLM limitations or primarily focus on modeling dependencies among decoded tokens, providing only partial improvements.

\par To fully address these limitations, we first conduct a theoretical analysis of DLMs. We model the expressive limitation of DLMs as a distribution shift between the diffusion process and the ground-truth. We find that this shift is governed by three factors: (i) the \emph{capacity factor}: the model's ability to fit the target distribution; (ii) the \emph{independency factor}: the model's ability to select independent tokens at each decoding step; and (iii) the \emph{invariance factor}: the model's ability to select tokens that can be totally determined by the current context and remain invariant to masked positions. Limitations in any of these factors can lead to distribution deviation and result in unreasonable outputs. We provide two simple examples in the top left panel of Figure~\ref{fig:framework}: (i) when the model lacks dependency awareness, it fails to capture relationships between tokens (e.g., "cat and mat" or "dog and bone"), leading to an implausible output such as "cat sits on the bone"; and (ii) when the model lacks invariance, it may decode "cat" without sufficient context, which should instead be "car". The first factor can be alleviated by increasing model scale, but a unified solution for all three factors is still lacking.

\par Based on this insight, we design an unified solution via energy function for all three factors. \textbf{First}, we focus on the invariance factor and generalize the re-weighting principle~\cite{zhou2022model} from causal inference to text generation. This principle encourages the model to learn invariant mappings by down-weighting non-invariant samples and emphasizing invariant ones. Specifically, we propose an invariant energy (Inv-E) function to implement this idea. Inv-E uses a proxy model to estimate possible future decoding over the entire sequence to capture the token invariance, where it assigns low energy to tokens that remain invariant under the current context and high energy otherwise. We further prove that this estimation can be achieved with a bounded error. \textbf{Second}, based on the principle of energy additivity~\cite{lecun2006tutorial,haarnoja2017reinforcement}, we integrate an existing energy function for the independency factor (Ind-E)~\cite{xu2024energy} with our Inv-E, resulting in an unified energy function (Uni-E). Uni-E is model-agnostic and can be parameterized at different scales using pretrained parameters (Uni-EDLM-AR) or NCE fine-tuning (Uni-EDLM-NCE). We also show that Uni-E has a unique advantage: its partition function can be computed analytically, enabling accurate and efficient energy computation without Markov Chain Monte Carlo (MCMC) sampling. \textbf{Third}, we theoretically analyze the gap to the ground-truth distribution and show that Uni-EDLM achieves a strictly smaller gap than existing DLMs in non-independent or non-invariant scenarios. To validate our method, we conduct experiments on multiple benchmark datasets for both DLMs and Diffusion Large Language Models (DLLMs). For DLMs, Uni-E consistently outperforms baselines and achieves around a 22\% speedup while maintaining generation quality comparable to AR models. For DLLMs, Uni-E improves performance across several benchmarks and significantly mitigates degradation as parallelism increases, consistently enhancing performance at different decoding speeds. The code is provided via link: \url{https://github.com/BlueWhaleLab/Uni-E}.

\section{Related Work}
\label{sec:related_work}
\noindent \textbf{Diffusion Language Models}. Diffusion Language Models (DLMs) are proposed as an alternative to auto-regressive (AR) text generation, allowing multiple tokens to be generated in parallel~\cite{li2025survey}. Early work mainly focused on extending diffusion process from continuous data to discrete text~\cite{austin2021structured,hoogeboom2021autoregressive,campbell2022continuous}, which led to the development of DLMs~\cite{lou2023discrete,austin2021structured,sahoo2024simple,arriola2025block}. Building on this, later studies scaled these models to larger sizes, developing Diffusion Large Language Models (DLLMs) that achieve performance close to Large Language Models (LLMs)~\cite{nie2025large,ye2025dream,bie2025llada2}. However, their performance still lags behind AR-based methods, especially when reducing the decoding steps. This is because they predict multiple tokens independently, making it hard to capture the relationships between them~\cite{li2026breaking}.

\noindent \textbf{Dependency-aware Parallel Decoding.} To address this, one line of research uses score guidance~\cite{ye2025dream,li2026breaking,luo2026dawn} or leverages a proxy model~\cite{liu2024discrete,xu2024energy,israel2025accelerating} to estimate token relationships. For example, EDLM~\cite{xu2024energy} introduces energy-based guidance from an AR model to compensate for dependency effects, while speculative decoding~\cite{israel2025accelerating} combines AR logits to improve generation. However, these methods mainly focus on dependency while overlooking another crucial factor: invariance. Another line of work designs remask strategy~\cite{wang2025remasking,yao2026remask,zhai2026core} to fix the decoding "error" in future steps, but these these methods are heuristic and unstable without theoretical guarantees, and generally require large decoding steps to perform well. We provide a detailed comparision with remask methods in Appendix~\ref{app:uni_e_comparison_remask_sec}. In this work, we provide a systematic analysis of the performance gap and propose a unified solution to capture all these factors.

\section{Preliminaries}
\textbf{Diffusion Language Model (DLM):} We denote the token sequence as $\mathbf{x}$ and we use subscript to denote the diffision time step and superscript to denote the position. For example, $\mathbf{x}_{t}^{i}$ represents the $i$-th token in the sequence at time step $t$. As for the diffusion process, we consider the time step range from 0 to 1 where $\mathbf{x}_{0}$ denotes the clean data and $\mathbf{x}_{1}$ denotes the fully noised data. Given discrete data $\mathbf{x}_{t-1}$ at time $t-1$, one diffusion step is defined as $q(\mathbf{x}_{t}|\mathbf{x}_{t-1}) \coloneqq Cat(\mathbf{x}_{t}; p = \mathbf{x}_{t-1} \mathbf{Q}_{t})$ where $ Cat(.;p) $ is a categorical distribution parameterized by $p$, and $\mathbf{Q}_{t}$ is the transition matrix. The forward process at time $t$ is defined as $q(\mathbf{x}_{t}|\mathbf{x}_{0}) \coloneqq Cat(\mathbf{x}_{t};p=\mathbf{x}_{0} \bar{\mathbf{Q}}_{t})$, where $\bar{\mathbf{Q}}_{t} \coloneqq \mathbf{Q}_{1} \mathbf{Q}_{2} ... \mathbf{Q}_{t}$. The reverse process at time $t$ is defined as: 
\begin{equation}
    q(\mathbf{x}_{t-1}|\mathbf{x}_{t},\mathbf{x}_{0}) \coloneqq \frac{q(\mathbf{x}_{t}|\mathbf{x}_{t-1},\mathbf{x}_{0})q(\mathbf{x}_{t-1}|\mathbf{x}_{0})}{q(\mathbf{x}_{t}|\mathbf{x}_{0})} = 
    Cat(\mathbf{x}_{t-1};p=\frac{\mathbf{x}_{t} \mathbf{Q}^{T}_{t} \odot \mathbf{x}_{0} \bar{\mathbf{Q}}_{t-1}}{\mathbf{x}_{0} \bar{\mathbf{Q}}_{t}\mathbf{x}^{T}_{t}}).
\label{eq:discrete_reverse}
\end{equation}
\par In this paper, we focus on masked diffusion language models (MDLMs) where the forward process replaces each unmasked token with a mask token $\mathbf{m}$ with probability $(1-\alpha_{t})$ at time $t$, thus the process between any two time steps $0<s<t<1$ is simplified as $q(\mathbf{x}_{t}|\mathbf{x}_{s}) = Cat(\mathbf{x}_{t};\alpha_{t/s}\mathbf{x}_{s} + (1-\alpha_{t/s})\mathbf{m})$ where $\alpha_{t/s}=\alpha_{t}/\alpha_{s}$. The backward process is computed from $q(\mathbf{x}_{t-1}|\mathbf{x}_{t},\mathbf{x}_{0})$ as:
\begin{equation}
q(\mathbf{x}_{s}|\mathbf{x}_{t}, \mathbf{x}_{0}) = 
\left\{
\begin{aligned}
    &Cat(\mathbf{x}_{s};\mathbf{x}_{t}), & \text{if } \mathbf{x}_{t} \neq \mathbf{m} \\
    &Cat(\mathbf{x}_{s};\frac{(1-\alpha_{s})\mathbf{m}+(\alpha_{s}-\alpha_{t})\mathbf{x}_{0}}{1-\alpha_{t}}). & \text{if } \mathbf{x}_{t} = \mathbf{m}
\end{aligned}
\right.
\label{eq:mdlm_reverse_true}
\end{equation}
MDLMs parameterize the backward process using a denoising model $\mathbf{f}_{\theta}(\mathbf{x}_{t}, t)$ to approximate $\mathbf{x}_{0}$ in Eq.~\ref{eq:mdlm_reverse_true}. In practice, they adopt an independence assumption: $\mathbf{f}_{\theta}(\mathbf{x}_{t}, t) = \prod_{l=1}^{L} \mathbf{f}_{\theta}(\mathbf{x}^{l}_{t}, t)$, which leads to $p_{\theta}(\mathbf{x}_{s}|\mathbf{x}_{t}) = \prod_{l=1}^{L} p_{\theta}(\mathbf{x}_{s}^{l}|\mathbf{x}_{t})$ where $L$ is the sequence length.

\textbf{Distribution of DLM:} Given a sequence $\mathbf{x}_{0}$, we consider $K$ intermediate time steps ($t=1,2,\ldots,K$), and the deriving decoding order is denoted as $\{Z_{1}, Z_{2}, ..., Z_{K} \}$. Each $Z_{t} = \{i|i \in \{1,2,...,L\}\}$ is an index set representing the decoding positions at step $t$. Based on the framework proposed in~\cite{lavenant2025error,chen2025optimal}, we model the distribution of DLM by two parts: a decoding planner $\mu_{\phi}$ and a distribution learner $p_{\theta}$. The decoding planner $\mu_{\phi}$ determines the decoding order, while the distribution learner $p_{\theta}$ predicts the token distribution. The overall DLM distribution across all possible decoding orders is defined as $P_{dlm}(\mathbf{x}_{0}) = \prod_{t=0}^{K} \mu_{\phi}(Z_{t}|Z_{>t}, \mathbf{x}_{t+1}) p_{\theta}(\mathbf{x}_{t}^{Z_{t}}|\mathbf{x}_{t+1})$, where $Z_{>t}$ is the union of $Z_{t+1}, Z_{t+2}, ..., Z_{K}$ and $\mathbf{x}_{t}^{Z_{t}}$ denotes the tokens of $\mathbf{x}_{t}$ positioned by $Z_{t}$. For a specific decoding order $\Psi = \{Z_{1}, Z_{2}, ..., Z_{K} \}$ the distribution learner is written as $p_{\theta}(\mathbf{x}_{0}, \Psi) = \prod_{t=0}^{K} p_{\theta}(\mathbf{x}_{t}^{Z_{t}}|\mathbf{x}_{t+1})$.
\section{Method}
\par In this section, we first conduct a theoretical analysis about the limitation of Diffusion Language Models (DLMs). We then propose the invariant energy (Inv-E) as the re-weighting function for invariant decoding. Next, we combine Inv-E with the existing independent energy (Ind-E) to obtain a unified energy (Uni-E), which can be easily parameterized and eliminates the need for partition estimation. Finally, we theoretically analyze the effectiveness of our method. The overall framework of our Uni-E is shown in Figure~\ref{fig:framework}.

\subsection{Expressive Limitation of DLMs}
\par To better understand the expressive limitations of DLMs, we analyze the distribution shift from the ground-truth distribution, defined as the Kullback–Leibler (KL) divergence between the model distribution and the ground-truth distribution:
\begin{equation}
    KL(\pi(\mathbf{x}_{0})||P_{dlm}(\mathbf{x}_{0})) = \mathbb{E}_{\pi(\mathbf{x}_{0})}[log \frac{\pi(\mathbf{x}_{0})}{P_{dlm}(\mathbf{x}_{0})}] = \mathbb{E}_{\pi(\mathbf{x}_{0}),\mu_{\phi}(\mathbf{x}_{0}, \Psi)}[log \frac{\pi(\mathbf{x}_{0})}{p_{\theta}(\mathbf{x}_{0},\Psi)}],
\label{eq:method_kl_div}
\end{equation}
where $\pi(.)$ denotes the ground-truth distribution. Given a specific order $\Psi = \{Z_{1}, Z_{2}, ..., Z_{K} \}$, the KL divergence can be expanded as:
\begin{equation}
\begin{aligned}
    & log \frac{\pi(\mathbf{x}_{0})}{p_{\theta}(\mathbf{x}_{0},\Psi)} = \sum_{t} log \frac{\pi(\mathbf{x}^{Z_{t}}_{t}|\mathbf{x}_{0}-\mathbf{x}^{Z_{t}}_{t})}{p_{\theta}(\mathbf{x}_{t}^{Z_{t}}|\mathbf{x}_{0}-\mathbf{x}^{Z_{t}}_{t})} = \sum_{t} \left[log \frac{\pi(\mathbf{x}^{Z_{t}}_{t}|\mathbf{x}_{0}-\mathbf{x}^{Z_{t}}_{t})}{\pi(\mathbf{x}^{Z_{t}}_{t}|\mathbf{x}_{t+1})} + log \frac{\pi(\mathbf{x}^{Z_{t}}_{t}|\mathbf{x}_{t+1})}{\prod_{i\in Z_{t}} \pi(\mathbf{x}^{i}_{t}|\mathbf{x}_{t+1})} \right. \\
    & \left. + log \frac{\prod_{i\in Z_{t}} \pi(\mathbf{x}^{i}_{t}|\mathbf{x}_{t+1})}{\prod_{i\in Z_{t}} p_{\theta}(\mathbf{x}^{i}_{t}|\mathbf{x}_{t+1})} + log \frac{\prod_{i\in Z_{t}} p_{\theta}(\mathbf{x}^{i}_{t}|\mathbf{x}_{t+1})}{p_{\theta}(\mathbf{x}_{t}^{Z_{t}}|\mathbf{x}_{t+1})} + log \frac{p_{\theta}(\mathbf{x}_{t}^{Z_{t}}|\mathbf{x}_{t+1})}{p_{\theta}(\mathbf{x}_{t}^{Z_{t}}|\mathbf{x}_{0}-\mathbf{x}^{Z_{t}}_{t})} \right],
\end{aligned}
\label{eq:method_kl_derive}
\end{equation}
where $\mathbf{x}_{0}-\mathbf{x}^{Z_{t}}_{t}$ denotes the relative complement of $\mathbf{x}^{Z_{t}}_{t}$ to $\mathbf{x}_{0}$. We find that the KL divergence is governed by the following three factors:

\noindent (i) \textbf{Capacity Factor}: This factor corresponds to the third term of the expanded KL divergence, which describes the model's capacity to fit the target distribution: 
\begin{equation}
    L_{lr} \coloneqq log \frac{\prod_{i\in Z_{t}} \pi(\mathbf{x}^{i}_{t}|\mathbf{x}_{t+1})}{\prod_{i\in Z_{t}} p_{\theta}(\mathbf{x}^{i}_{t}|\mathbf{x}_{t+1})} = \sum_{i \in Z_{t}} log \frac{\pi(\mathbf{x}^{i}_{t}|\mathbf{x}_{t+1})}{p_{\theta}(\mathbf{x}^{i}_{t}|\mathbf{x}_{t+1})}.
\label{qe:method_lr_limit}
\end{equation}

\begin{figure}[t!]
\vspace{-15pt}
\centering
\includegraphics[width=0.95\textwidth]{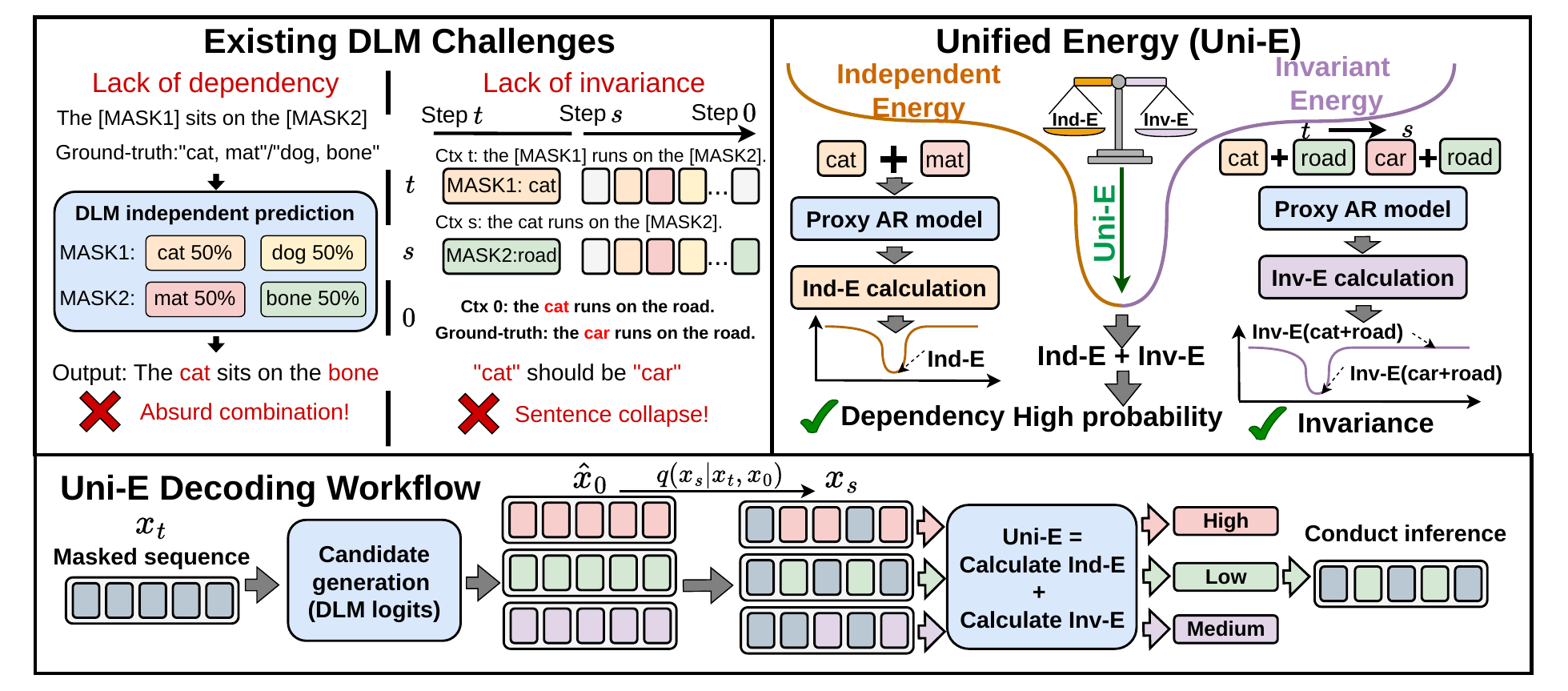}
\caption{Challenges of Diffusion Language Models and the framework of our unified energy (Uni-E). Uni-E uses a proxy model to capture both dependency and invariance to guide the decoding process.}
\label{fig:framework}
\vspace{-15pt}
\end{figure}
\noindent (ii) \textbf{Dependency Factor}: This factor corresponds to the second and fourth terms:
\begin{equation}
    L_{dp} \coloneqq log \frac{\pi(\mathbf{x}_{t}^{Z_{t}}|\mathbf{x}_{t+1})}{p_{\theta}(\mathbf{x}_{t}^{Z_{t}}|\mathbf{x}_{t+1})} + log \frac{\prod_{i\in Z_{t}} p_{\theta}(\mathbf{x}^{i}_{t}|\mathbf{x}_{t+1})}{\prod_{i\in Z_{t}} \pi(\mathbf{x}^{i}_{t}|\mathbf{x}_{t+1})} = log \frac{\pi(\mathbf{x}_{t}^{Z_{t}}|\mathbf{x}_{t+1})}{\prod_{i\in Z_{t}} \pi(\mathbf{x}^{i}_{t}|\mathbf{x}_{t+1})} + log \frac{\prod_{i\in Z_{t}} p_{\theta}(\mathbf{x}^{i}_{t}|\mathbf{x}_{t+1})}{p_{\theta}(\mathbf{x}_{t}^{Z_{t}}|\mathbf{x}_{t+1})}.
\label{eq:method_dp_limit}
\end{equation}
One way to reduce this factor is to simultaneously reduce: i) the divergence between $p_{\theta}(\mathbf{x}_{t}^{Z_{t}}|\mathbf{x}_{t+1})$ and $\pi(\mathbf{x}_{t}^{Z_{t}}|\mathbf{x}_{t+1})$; and ii) the divergence between $p_{\theta}(\mathbf{x}^{i}_{t}|\mathbf{x}_{t+1})$ and $\pi(\mathbf{x}^{i}_{t}|\mathbf{x}_{t+1})$ for all positions. However, this is almost infeasible because modeling the joint distribution $p_{\theta}(\mathbf{x}_{t}^{Z_{t}}|\mathbf{x}_{t+1})$ is difficult, and sampling from $\pi(\mathbf{x}^{Z_{t}}_{t}|\mathbf{x}_{t+1})$ is intractable. An alternative way is to identify independent tokens that can be decoded in one denoising step, such that $\pi(\mathbf{x}_{t}^{Z_{t}}|\mathbf{x}_{t+1}) = \prod_{i\in Z_{t}} \pi(\mathbf{x}^{i}_{t}|\mathbf{x}_{t+1})$, and a well-optimized $p_{\theta}$ should also satisfy $p_{\theta}(\mathbf{x}_{t}^{Z_{t}}|\mathbf{x}_{t+1}) = \prod_{i\in Z_{t}} p_{\theta}(\mathbf{x}^{i}_{t}|\mathbf{x}_{t+1})$. Therefore, this factor is totally bounded by the independency of $Z_{t}$.

\noindent (iii) \textbf{Invariance Factor}: This factor corresponds to the first and last terms:
\begin{equation}
    L_{iv} \coloneqq log \frac{\pi(\mathbf{x}_{t}^{Z_{t}}|\mathbf{x}_{0}-\mathbf{x}_{t}^{Z_{t}})}{p_{\theta}(\mathbf{x}_{t}^{Z_{t}}|\mathbf{x}_{0}-\mathbf{x}_{t}^{Z_{t}})} + log \frac{p_{\theta}(\mathbf{x}_{t}^{Z_{t}}|\mathbf{x}_{t+1})}{\pi(\mathbf{x}_{t}^{Z_{t}}|\mathbf{x}_{t+1})} = log \frac{\pi(\mathbf{x}_{t}^{Z_{t}}|\mathbf{x}_{0}-\mathbf{x}_{t}^{Z_{t}})}{\pi(\mathbf{x}_{t}^{Z_{t}}|\mathbf{x}_{t+1})} + log \frac{p_{\theta}(\mathbf{x}_{t}^{Z_{t}}|\mathbf{x}_{t+1})}{p_{\theta}(\mathbf{x}_{t}^{Z_{t}}|\mathbf{x}_{0}-\mathbf{x}_{t}^{Z_{t}})}.
\label{method:iv_limit}
\end{equation}
Due to the same reasons of sampling and joint distribution modeling, we require that the decoded tokens at each step can be fully determined. This leads to $\pi(\mathbf{x}_{t}^{Z_{t}}|\mathbf{x}_{0}-\mathbf{x}_{t}^{Z_{t}}) = \pi(\mathbf{x}_{t}^{Z_{t}}|\mathbf{x}_{t+1})$. And a well-optimized model should satisfy $p_{\theta}(\mathbf{x}_{t}^{Z_{t}}|\mathbf{x}_{0}-\mathbf{x}_{t}^{Z_{t}}) = p_{\theta}(\mathbf{x}_{t}^{Z_{t}}|\mathbf{x}_{t+1})$. This factor describes invariance in decoding: the distribution of decoded tokens (e.g. $\mathbf{x}_{t}^{Z_{t}}$) should be fully determined by the current context (e.g. $\mathbf{x}_{t+1}$) and remain unchanged in subsequent steps.

\subsection{Energy Function for Invariant Decoding}
In order to address the expressive limitations of DLMs, we first propose an effective way to guide invariant decoding for the invariance factor. Our solution is inspired by the re-weighting principle from causal inference~\cite{zhou2022model}, whose idea is that models can learn invariant representations across different feature parts (e.g., objects and background) by reweighting samples using clustering, graph partitioning, or neural network guidance. We generalize this principle to text generation in two ways: (i) we treat tokens in the sequence as basic units and consider their invariance, and (ii) we adjust the token distribution using a re-weighting function to ensure that the posterior distribution of the current decoding is close to the posterior distribution of its invariant part. Specifically, for the $i$-th token at step $t$, we assume its invariant part is $\mathcal{V}_{i}$, satisfying $\pi(\mathbf{x}^{i}_{t} | \mathbf{x}_{0} - \mathbf{x}_{0}^{i}) = \pi(\mathbf{x}_{t}^{i} | \mathbf{x}_{0}^{\mathcal{V}_{i}})$. This re-weighting function enforces $\mathcal{W}(\mathbf{x}_{t}^{i}, \mathbf{x}_{t+1}) p_{\theta}(\mathbf{x}_{t}^{i} | \mathbf{x}_{t+1}) = \pi(\mathbf{x}_{t}^{i} | \mathbf{x}_{0}^{\mathcal{V}_{i}})$. We define the current context intersection as $\mathcal{V}_{i}^{+} \coloneqq \mathcal{V}_{i} \cap Z_{>t}$ and the current context difference as $\mathcal{V}_{i}^{-} \coloneqq \mathcal{V}_{i} - Z_{>t}$. The re-weighting function can then be derived as:
\begin{equation}
\begin{aligned}
    \mathcal{W}(\mathbf{x}_{t}^{i}, \mathbf{x}_{t+1}) & = \frac{\pi(\mathbf{x}_{t}^{i}|\mathbf{x}_{0}^{\mathcal{V}_{i}})}{p_{\theta}(\mathbf{x}_{t}^{i} | \mathbf{x}_{t+1})} = \frac{\pi(\mathbf{x}_{t}^{i}|\mathbf{x}_{0}^{\mathcal{V}^{+}_{i}}, \mathbf{x}_{0}^{\mathcal{V}^{-}_{i}})}{p_{\theta}(\mathbf{x}_{t}^{i} | \mathbf{x}_{t+1})} = \frac{\pi(\mathbf{x}_{0}^{i}, \mathbf{x}_{0}^{\mathcal{V}^{-}_{i}} |\mathbf{x}_{0}^{\mathcal{V}^{+}_{i}})}{\pi(\mathbf{x}_{0}^{i}|\mathbf{x}_{0}^{\mathcal{V}^{+}_{i}})  \pi(\mathbf{x}_{0}^{\mathcal{V}^{-}_{i}}|\mathbf{x}_{0}^{\mathcal{V}^{+}_{i}})} \frac{\pi(\mathbf{x}_{0}^{i}|\mathbf{x}_{0}^{\mathcal{V}^{+}_{i}})}{p_{\theta}(\mathbf{x}_{t}^{i} | \mathbf{x}_{t+1})}. \\
\end{aligned}
\label{eq:method_reweight}
\end{equation}
Here, $\mathbf{x}^{\mathcal{V}^{+}_{i}}$ is difficult to obtain directly, so we use a proxy AR-based model $p_{AR}(.|\mathbf{x}_{t+1})$ to estimate $\pi(.|\mathbf{x}^{\mathcal{V}^{+}_{i}})$. The reasons behind are: (i) $\mathbf{x}_{t}^{i}$ remains invariant to the unmasked tokens in $\mathbf{x}_{t+1}$ except those in $\mathcal{V}^{+}_{i}$, so $\mathbf{x}_{t+1}$ can serve as a valid conditioning context; (ii) an AR model pretrained on clean data $\mathbf{x}_{0}$ is well-suited to model this distribution, and its posterior can be computed efficiently as $p_{AR}(\mathbf{x}_{0}|\mathbf{x}_{t}) = p_{AR}(\mathbf{x}_{t}, \mathbf{x}_{0} - \mathbf{x}_{t}) / p_{AR}(\mathbf{x}_{t}) =  p_{AR}(\mathbf{x}_{0}) / p_{AR}(\mathbf{x}_{t})$ where $p_{AR}(\mathbf{x}_{t}) = \sum_{Z_{t}} p_{AR}(\mathbf{x}_{0}^{Z_{t}}, \mathbf{x}_{0} - \mathbf{x}_{t})$ since the unmasked tokens are fixed and unchanged. We represent the re-weighting function as an energy function, called invariant energy (Inv-E), which is defined as: $\mathcal{W}(\textbf{x}_{t}^{i}, \textbf{x}_{t+1}) = \frac{exp(-E_{inv}(\mathbf{x}_{t}^{i}, \mathbf{x}_{t+1}))}{F(\mathbf{x}_{t})}$. Based on Eq.~\ref{eq:method_reweight} we obtain the following energy formulation:
\begin{equation}
\begin{aligned}
    E_{inv}(\mathbf{x}_{t}^{i}, \mathbf{x}_{t+1}) = log \frac{p_{\theta}(\mathbf{x}_{t}^{i}|\mathbf{x}_{t+1})}{p_{AR}(\mathbf{x}_{0}^{i}|\mathbf{x}_{t+1})} + log \frac{p_{AR}(\mathbf{x}_{0}^{\mathcal{V}^{-}_{i}}|\mathbf{x}_{t+1})}{p_{AR}(\mathbf{x}_{0}^{\mathcal{V}^{-}_{i}}|\mathbf{x}_{0}^{i}, \mathbf{x}_{t+1})} - log F(\mathbf{x}_{t}).
\end{aligned}
\label{qd:method_inv_energy_formulation}
\end{equation}
Another question is that the term $log \frac{p_{AR}(\mathbf{x}_{0}^{\mathcal{V}^{-}_{i}}|\mathbf{x}_{t+1})}{p_{AR}(\mathbf{x}_{0}^{\mathcal{V}^{-}_{i}}|\mathbf{x}_{0}^{i},\mathbf{x}_{t+1})}$ cannot be calculated directly, since the difference part $\mathcal{V}^{-}_{i}$ is unknown. Fortunately, it can be effectively estimated through Lemma~\ref{theory:lemma_1}. The detailed proof is provided in Appendix~\ref{app:lemma_proof}.

\begin{lemma}[Invariance Sampling Estimation]
The invariance divergence term can be estimated by sampling $Z_{<t}$. Given the sampled index set is $Z_{s}$, the average estimation over all possible $Z_{s}$ satisfies the following property:
\begin{equation}
\begin{aligned}
    \mathbb{E}_{Z_{s}} \left\{ log \frac{p_{AR}(\mathbf{x}_{0}^{Z_{s}}|\mathbf{x}_{t+1})}{p_{AR}(\mathbf{x}_{0}^{Z_{s}}|\mathbf{x}^{i}_{0}, \mathbf{x}_{t+1})} \right\} = \beta log \frac{\pi(\mathbf{x}_{0}^{\mathcal{V}^{-}_{i}}|\mathbf{x}_{t+1})}{\pi(\mathbf{x}_{0}^{\mathcal{V}^{-}_{i}}|\mathbf{x}^{i}_{0}, \mathbf{x}_{t+1})} + \gamma,
\end{aligned}
\end{equation}
where $\beta$ is the approximation degree falling into $\left(0,1\right]$ and $\gamma$ is the approximation error between $p_{AR}(.)$ and $\pi(.)$. If the sampling number is $r$ and the index length of $\mathcal{V}^{-}_{i}$ and $Z_{<t} - \mathcal{V}^{-}_{i}$ is $x$ and $y$ respectively, then the average $\beta$ is $\frac{r}{x+y}$.
\label{theory:lemma_1}
\end{lemma}

For each time step $t$, we randomly select another time step $s$ satisfying $0<s<t<1$. The forward diffusion produces two noised data $\mathbf{x}_{s}$ and $\mathbf{x}_{t}$, where $\mathbf{x}_{t}$ carries all masked tokens from $\mathbf{x}_{s}$. (i.e., $\mathbf{x}_{0} \xrightarrow{mask} \mathbf{x}_{s} \xrightarrow{mask} \mathbf{x}_{t}$). We treat $\mathbf{x}_{s}$ as the "next step" of $\mathbf{x}_{t}$ and the decoding of the masked tokens in $\mathbf{x}_{s}$ is regarded as "future decoding"(i.e., invariance sampling $Z_{s}$). Based on the energy formulation in Eq.~\ref{qd:method_inv_energy_formulation} and Lemma~\ref{theory:lemma_1}, Inv-E can be estimated as:
\begin{equation}
\begin{aligned}
    E_{inv}(\mathbf{x}_{0}, \mathbf{x}_{s}, \mathbf{x}_{t}) = log \frac{p_{\theta}(\mathbf{x}_{s}|\mathbf{x}_{t})}{p_{AR}(\mathbf{x}_{s}|\mathbf{x}_{t})} + log \frac{p_{AR}(\mathbf{x}_{0}|\mathbf{x}_{t})}{p_{AR}(\mathbf{x}_{0}|\mathbf{x}_{s})} - log F_{inv} (\mathbf{x}_{t}),
\end{aligned}
\label{eq:inv_energy}
\end{equation}
where $F_{inv} (\mathbf{x}_{t})$ is the partition function. Intuitively, the first term $log \frac{p_{\theta}(\mathbf{x}_{s}|\mathbf{x}_{t})}{p_{AR}(\mathbf{x}_{s}|\mathbf{x}_{t})}$ corresponding to term $log \frac{p(\mathbf{x}_{t}^{i}|\mathbf{x}_{t+1})}{p_{AR}(\mathbf{x}_{0}^{i}|\mathbf{x}_{t+1})}$ in Eq.~\ref{qd:method_inv_energy_formulation} describes the model ability to predict the next step distribution. The second term $log \frac{p_{AR}(\mathbf{x}_{0}|\mathbf{x}_{t})}{p_{AR}(\mathbf{x}_{0}|\mathbf{x}_{s})}$ acts as the \textit{Invariance Sampling Estimation} describing the invariance, since if the decoding is invariant it will have no impact on the rest tokens (i.e., $p_{AR}(\mathbf{x}_{0}|\mathbf{x}_{t}) \approx p_{AR}(\mathbf{x}_{0}|\mathbf{x}_{s})$).


\subsection{Unified Energy for Independent and Invariant Decoding}
Based on our Inv-E, we further propose an unified energy (Uni-E) for all the factors. For the dependency factor, we adopt the independent energy (Ind-E) proposed in~\cite{xu2024energy}, which is defined as:
\begin{equation}
\begin{aligned}
    E_{ind}(\mathbf{x}_{0}, \mathbf{x}_{t}) = log \frac{p_{\theta}(\mathbf{x}_{0}|\mathbf{x}_{t})}{p_{AR}(\mathbf{x}_{0}|\mathbf{x}_{t})} - log F_{ind} (\mathbf{x}_{t}),
\end{aligned}
\label{eq:id_energy}
\end{equation}
where $F_{ind} (\mathbf{x}_{t})$ is the partition function. The idea behind Ind-E is that the independence of parallel decoded tokens can be captured by an AR model, as it operates on the entire sequence and decodes tokens sequentially. More details can be found in~\cite{xu2024energy}. We argue that Ind-E can be seamlessly integrated with Inv-E for two reasons: (i) Ind-E acts as the residual term for the DLM distribution, which has a similar form to the re-weighting function used in Inv-E; and (ii) energy functions can be naturally combined due to the principle of additivity~\cite{lecun2006tutorial}. This leads to the unified energy (Uni-E):
\begin{equation}
\begin{aligned}
    E_{uni}(\mathbf{x}_{0}, \mathbf{x}_{s}, \mathbf{x}_{t}) = E_{inv} + E_{ind} = log \frac{p_{\theta}(\mathbf{x}_{s}|\mathbf{x}_{t})}{p_{AR}(\mathbf{x}_{s}|\mathbf{x}_{t})} + log \frac{p_{\theta}(\mathbf{x}_{0}|\mathbf{x}_{t})}{p_{AR}(\mathbf{x}_{0}|\mathbf{x}_{s})} - log F (\mathbf{x}_{t}),
\end{aligned}
\label{eq:uni_energy}
\end{equation}
where $F(\mathbf{x}_{t})$ is the new partition function. This unified formulation simultaneously captures both dependency and invariance through the term $log \frac{p_{\theta}(\mathbf{x}_{0}|\mathbf{x}_{t})}{p_{AR}(\mathbf{x}_{0}|\mathbf{x}_{s})}$: the AR guidance models the dependency relationships (i.e., $p_{\theta}(\mathbf{x}_{0}|\mathbf{x}_{t}) \approx p_{AR}(\mathbf{x}_{0}|\mathbf{x}_{t})$), while the decoding from $\mathbf{x}_{s}$/$\mathbf{x}_{t}$ reflects the invariance (i.e., $p_{AR}(\mathbf{x}_{0}|\mathbf{x}_{t}) \approx p_{AR}(\mathbf{x}_{0}|\mathbf{x}_{s})$). We parameterize the unified energy-based model (Uni-EBM) as:
\begin{equation}
\begin{aligned}
    E_{\phi}(\mathbf{x}_{0}, \mathbf{x}_{s}, \mathbf{x}_{t}) = log \frac{p_{\theta}(\mathbf{x}_{s}|\mathbf{x}_{t})}{p_{AR}(\mathbf{x}_{s}|\mathbf{x}_{t})} + log \frac{p_{\theta}(\mathbf{x}_{0}|\mathbf{x}_{t})}{p_{AR}(\mathbf{x}_{0}|\mathbf{x}_{s})},
\end{aligned}
\label{eq:uni_ebm}
\end{equation}
\noindent \textbf{Partition Function Estimation}: Estimating the partition function is notoriously intractable. Traditional methods like Markov Chain Monte Carlo (MCMC) sampling or importance sampling require a large number of samples, which is not only inefficient but also inaccurate. Here, we show that \textit{the partition function of $E_{uni}$ can be calculated analytically, avoiding intractable sampling and estimation}. Based on the definition $log F(\mathbf{x}_{t}) = log \sum_{\mathbf{x}_{0}} p_{\theta}(\mathbf{x}_{0}|\mathbf{x}_{t}) exp(-E_{\phi}(\mathbf{x}_{0}, \mathbf{x}_{s}, \mathbf{x}_{t}))$, we derive it as:
\begin{equation}
\begin{aligned}
    logF(\mathbf{x}_{t}) & = log \sum_{\mathbf{x}_{0}} p_{\theta}(\mathbf{x}_{0}|\mathbf{x}_{t}) \frac{p_{AR}(\mathbf{x}_{s}|\mathbf{x}_{t})}{p_{\theta}(\mathbf{x}_{s}|\mathbf{x}_{t})}\frac{p_{AR}(\mathbf{x}_{0}|\mathbf{x}_{s})}{p_{\theta}(\mathbf{x}_{0}|\mathbf{x}_{t})} = log \left\{\sum_{\mathbf{x}_{0}} p_{AR}(\mathbf{x}_{0}|\mathbf{x}_{s})\right\} + log \frac{p_{AR}(\mathbf{x}_{s}|\mathbf{x}_{t})}{p_{\theta}(\mathbf{x}_{s}|\mathbf{x}_{t})}, \\
\end{aligned}
\label{eq:partition_calculation}
\end{equation}

where $\sum_{\mathbf{x}_{0}} p_{AR}(\mathbf{x}_{0}|\mathbf{x}_{s}) = 1$, and the partition function is $log \frac{p_{AR}(\mathbf{x}_{s}|\mathbf{x}_{t})}{p_{\theta}(\mathbf{x}_{s}|\mathbf{x}_{t})}$. This closed-form partition provides a unique advantage: Uni-EBM avoids the need for intractable estimation and enables accurate computation of the evidence lower bound (ELBO). After applying the Uni-EBM $E_{\phi}(\mathbf{x}_{0},\mathbf{x}_{s},\mathbf{x}_{t})$ to the DLM, we obtain the unified energy-based DLM (Uni-EDLM) as:
\begin{equation}
\begin{aligned}
    p_{\phi,\theta}(\textbf{x}_{0} | \textbf{x}_{t}) =  \frac{exp(-E_{\phi}(\mathbf{x}_{0}, \mathbf{x}_{s}, \mathbf{x}_{t}))}{F(\mathbf{x}_{t})} p_{\theta}(\textbf{x}_{0} | \textbf{x}_{t}).
\end{aligned}
\label{eq:uni_edlm}
\end{equation}
\noindent \textbf{Uni-EDLM Parameters}: Similar to the independent energy-based diffusion language model~\cite{xu2024energy}, there are two ways to obtain the model parameters of Uni-EDLM. The first is to utilize a pretrained AR model and a DLM to directly calculate $E_{\phi}(\mathbf{x}_{0},\mathbf{x}_{s},\mathbf{x}_{t})$, and we refer to the resulting model as Uni-EDLM-AR. The second one is to fine-tune the pretrained parameters using the Noise Contrastive Estimation (NCE) loss~\cite{gutmann2010noise}. In this setting, we freeze the pretrained AR model and fine-tune the parameters of the DLM. We use the clean data $\mathbf{x}_{0}$ together with $\mathbf{x}_{s}$, $\mathbf{x}_{t}$ as positive samples, and we sample $\hat{\mathbf{x}}_{0}$ from $p_{\theta}(\textbf{x}_{0} | \textbf{x}_{t})$ and derive $\hat{\mathbf{x}}_{s}$ as negative samples. We refer to the fine-tuned model as Uni-EDLM-NCE. The loss is described as follows, and the training algorithm is in Appendix~\ref{app:training_alg}.
\begin{equation}
\begin{aligned}
    \mathcal{L}_{NCE}(\phi,\theta) & = \mathbb{E}_{\mathbf{x}_{0},\mathbf{x}_{t} \sim \pi, q} \{\mathbb{E}_{\mathbf{x}_{s} \sim q(.|\mathbf{x}_{t},\mathbf{x}_{0})} log\frac{1}{1+exp(E_{\phi}(\mathbf{x}_{0}, \mathbf{x}_{s}, \mathbf{x}_{t}))} \\ 
    & + \mathbb{E}_{\hat{\mathbf{x}}_{0}, \hat{\mathbf{x}}_{s} \sim p_{\theta}(.|\mathbf{x}_{t}), q(.|\mathbf{x}_{t},\hat{\mathbf{x}}_{0})} log\frac{1}{1+exp(-E_{\phi}(\hat{\mathbf{x}}_{0}, \hat{\mathbf{x}}_{s}, \mathbf{x}_{t}))}\}.
\end{aligned}
\label{eq:nce_loss}
\end{equation}
\noindent \textbf{Inference of Uni-EDLM}: We adapt the self-normalized importance sampling strategy~\citep{hammersley2013monte,james1980monte,xu2024energy} to Uni-EDLM, whose framework is shown at the bottom of Figure~\ref{fig:framework}. At each time step $t$, we randomly select time step $s<t$. The importance sampling involves: (i) sampling multiple $\mathbf{x}_{0}$ candidates $\{\mathbf{x}_{0}^{j}\}_{j=1}^{k}$ from the DLM in parallel; (ii) sampling $\mathbf{x}_{s}^{j}$ from $\mathbf{x}^{j}_{0}$ and $\mathbf{x}_{t}$ in parallel via the backward process in Eq.~\ref{eq:mdlm_reverse_true}; (iii) feeding these tuples $\{(\mathbf{x}_{0}^{j},\mathbf{x}_{s}^{j}, \mathbf{x}_{t})\}_{j=1}^{k}$ into the energy function in batch to calculate their energies; and (iv) selecting the $\mathbf{x}^{j}_{0}$ with the lowest energy. The selected $\mathbf{x}^{j}_{0}$ is then fed into $q(\mathbf{x}_{t-1}|\mathbf{x}_{t}, \mathbf{x}_{0})$ to conduct one-step inference. In practice, we also define a sampling window with size $w$ ($w=0.2$ by default), and apply importance sampling only within the interval $(1-w, 1]$. The full inference algorithm is provided in Appendix~\ref{app:sampling_alg}.
 
\noindent \textbf{Effectiveness Discussion}: We analyze the effectiveness of Uni-EDLM with respect to the three limitation factors. First we ask: \textit{Can Uni-EDLM reduce the model distribution shift in non-independent or non-invariant decoding scenarios?} Our answer is \textbf{yes}. The effectiveness of Uni-EDLM on dependency and invariance factors is formalized in Lemma~\ref{theory:lemma_2}, with the proof provided in Appendix~\ref{app:lemma_2_proof}. Besides, Uni-EBM can be integrated with any scale models to obtain Uni-EDLMs with various size. This addresses the limitation related to the capacity factor.
\begin{lemma}[Effectiveness of Uni-EDLM]
Across all the possible decodings, the Kullback-Leibler divergence between Uni-EDLM and the ground-truth distribution is strictly lower than that between DLM and the ground-truth distribution in non-independent or non-invariant scenarios.
\label{theory:lemma_2}
\end{lemma}


\section{Experiments}

\subsection{Experiments on Diffusion Language Models (DLMs)}
\label{subsec:dlm_exp}
\noindent \textbf{Settings}: Following prior works~\cite{xu2024energy,sahoo2025diffusion,sahoo2024simple,arriola2025block}, we use OpenWebText (OWT)~\cite{Gokaslan2019OpenWeb} for fine-tuning and evaluation, along with seven zero-shot datasets: PTB~\cite{marcus1993building}, Wikitext~\cite{merity2016pointer}, LM1B~\cite{chelba2013one}, Lambada~\cite{paperno2016lambada}, AG News~\cite{zhang2015character}, Pubmed and Arxiv~\cite{cohan2018discourse}. The baselines include: (i) Transformer (AR)~\cite{vaswani2017attention}; (ii) traditional DLMs: MDLM~\cite{sahoo2024simple}, SEDD~\cite{lou2023discrete}, Duo~\cite{sahoo2025diffusion} and Block Diffusion~\cite{arriola2025block}; (iii) dependency-aware DLMs: EDLM-(AR and NCE)~\cite{xu2024energy} and DCD~\cite{liu2024discrete}; and (iv) remask: ReMDM~\cite{wang2025remasking}. Perplexity (PPL) and Generative Perplexity (Gen PPL) are used to measure effectiveness and we follow~\cite{xu2024energy} to use average wall-clock time to evaluate efficiency. More details are provided in Appendix~\ref{app:dlm_settings}.

\noindent \textbf{Implementation Details}: For the AR models, we follow common practice and use GPT-2-small~\cite{radford2019language,sahoo2024simple,xu2024energy} for our method and baselines. We use the pretrained MDLM~\cite{sahoo2024simple} as the diffusion backbone. These pretrained models are used directly for Uni-EDLM-AR. For Uni-EDLM-NCE, we freeze the AR model and fine-tune the MDLM parameters using 4xH200 GPUs. The maximum fine-tuning step is 1,000,000, with a learning rate of 3e-4. The per-GPU batch size is 16. We save checkpoint every 10,000 steps and select the best model based on the dev set (randomly select 0.1\% data). We set the candidate size $k$ as 2 for importance sampling and set the window size $w$ as 0.2 by default.

\begin{table}[h!]
\centering
\small
\caption{Perplexity of all the methods. The Lower is better. $^{*}$ denotes reproduced results, $^{\dagger}$ taken from~\cite{xu2024energy}, $^{\ddagger}$ taken from~\cite{sahoo2025diffusion}. The best result is in \textbf{bold} and the second best is underlined (except AR).}
\begin{tabular}{l|c|ccccccc}
\toprule
Methods & OWT & PTB & WikiText & LM1B & Lambada & AG News & Pubmed & Arxiv \\
\midrule
AR$^{\dagger}$ & 17.56 & 82.05 & 25.75 & 51.25 & 51.28 & 52.09 & 49.01 & 41.73 \\
\midrule
SEDD$^{\dagger}$ & 24.56 & 100.09 & 34.28 & 68.20 & 49.86 & 62.09 & 44.53 & 38.48 \\
MDLM$^{\dagger}$ & 23.83 & 95.26 & 32.83 & 67.01 & 47.52 & 61.15 & 41.89 & 37.37 \\
Duo$^{\ddagger}$ & 23.25 & 89.35 & 33.57 & 73.86 & 49.78 & 67.81 & 44.48 & 40.39 \\
Block Diffusion$^{*}$ & 21.27 & 97.23 & 31.82 & 61.39 & 49.50 & 62.16 & 42.18 & 38.74 \\
\midrule
DCD$^{*}$ & 23.33 & 95.69 & 35.20 & 66.51 & \underline{46.71} & 61.35 & 47.65 & 40.09 \\
EDLM-NCE$^{\dagger}$ & 21.52 & 93.21 & 30.77 & 63.19 & 46.92 & 60.02 & \underline{41.80} & \underline{36.63} \\
EDLM-CoAR$^{\dagger}$ & 17.58 & 89.73 & 28.31 & \underline{60.23} & 50.04 & \textbf{57.94} & 46.31 & 39.02 \\
\midrule
Uni-EDLM-AR & \underline{17.41} & \underline{87.36} & \underline{26.62} & 60.51 & \textbf{46.62} & \underline{58.40} & \textbf{41.35} & 37.37 \\
Uni-EDLM-NCE & \textbf{16.54} & \textbf{86.18} & \textbf{24.01} & \textbf{58.14} & 48.26 & 58.66 & 42.74 & \textbf{35.53} \\
\bottomrule
\end{tabular}
\label{exp:tab_ppl_comparison}
\end{table}

\begin{table*}[h!]
\centering
\tiny
\setlength{\tabcolsep}{4pt}
\caption{Generative Perplexity under Llama2, Llama3, and GPT2. $\downarrow$ means the lower is better. The best results are highlighted with \textbf{bold}. $^{*}$ denotes reproduced results, $^{\dagger}$ taken from~\cite{xu2024energy}.}
\begin{tabular}{l|c|ccc|c}
\toprule
Method & Timesteps & Llama2$\downarrow$ & Llama3$\downarrow$ & GPT2$\downarrow$ & Entropy \\
\midrule
Data & - & 7.0 & 9.4 & 14.7 & 7.7 \\
AR$^{\dagger}$ & 1024 & 22.9 & 40.3 & 35.7 & 8.1 \\
\midrule
SEDD$^{\dagger}$ & 128 / 256 / 512 / 1024 
& 46.3 / 36.1 / 32.5 / 27.3
& 86.6 / 65.0 / 54.3 / 43.7
& 78.5 / 64.8 / 52.2 / 41.5
& 8.1 / 7.8 / 7.7 / 7.6 \\

MDLM$^{\dagger}$ & 128 / 256 / 512 / 1024 
& 48.5 / 37.2 / 30.6 / 27.6
& 86.7 / 66.9 / 52.6 / 44.6
& 78.8 / 66.8 / 52.4 / 42.6
& 8.1 / 7.9 / 7.8 / 7.6 \\

Duo$^{*}$ & 128 / 256 / 512 / 1024 
& 45.8 / 35.6 / 29.5 / 25.4
& 85.3 / 65.4 / 51.9 / 43.7
& 74.1 / 63.2 / 48.6 / 36.6
& 7.9 / 7.7 / 7.6 / 7.6 \\

Block Diffusion$^{*}$ & 128 / 256 / 512 / 1024 
& 47.1 / 37.9 / 31.1 / 26.7
& 89.8 / 69.6 / 54.1 / 45.0
& 83.9 / 66.7 / 51.8 / 40.8
& 7.7 / 7.5 / 7.3 / 7.1 \\

\midrule
DCD$^{*}$ & 128 / 256 / 512 / 1024 
& 43.8 / 36.6 / 28.5 / 22.2
& 82.4 / 64.1 / 48.4 / 40.9
& 74.3 / 63.5 / 47.3 / 38.7
& 7.9 / 7.6 / 7.6 / 7.5 \\

ReMDM$^{*}$ & 128 / 256 / 512 / 1024 
& 47.2 / 35.4 / 29.6 / 21.8
& 86.4 / 64.0 / 42.8 / 33.0
& 76.4 / 64.1 / 44.6 / 26.9
& 8.1 / 7.7 / 7.7 / 7.5 \\

EDLM-AR$^{\dagger}$ & 128 / 256 / 512 / 1024 
& 43.2 / 34.7 / 26.8 / 19.0
& 83.2 / 62.2 / 44.4 / 28.8
& 71.3 / 62.1 / 42.0 / 25.5
& 8.0 / 7.9 / 7.6 / 7.2 \\

EDLM-NCE$^{\dagger}$ & 128 / 256 / 512 / 1024 
& 43.2 / 35.7 / 26.3 / 19.6
& 83.5 / 62.9 / 44.1 / 28.8
& 71.7 / 61.7 / 42.5 / 25.5 
& 8.0 / 7.9 / 7.6 / 7.3 \\
\midrule
Uni-EDLM-AR & 128 / 256 / 512 / 1024 
& \textbf{40.4} / 32.4 / 23.2 / 13.4
& \textbf{79.9} / \textbf{56.5} / 37.6 / 19.2
& \textbf{68.2} / 55.5 / 37.6 / 19.2
& 8.1 / 7.9 / 7.5 / 7.2 \\

Uni-EDLM-NCE & 128 / 256 / 512 / 1024 
& 40.6 / \textbf{32.2} / \textbf{22.8} / \textbf{13.1}
& 80.0 / 57.3 / \textbf{37.5} / \textbf{18.9}
& 69.4 / \textbf{54.3} / \textbf{37.5} / \textbf{18.9}
& 8.1 / 7.9 / 7.5 / 7.2 \\
\bottomrule

\end{tabular}
\label{exp:tab_gen_ppl}
\end{table*}

\textbf{Perplexity (PPL)}: The PPL results are shown in Table~\ref{exp:tab_ppl_comparison}. Uni-EDLM significantly outperforms traditional DLMs. In particular, compared to the backbone MDLM, it achieves a 26.9\% improvement on OWT without fine-tuning. After fine-tuning, Uni-EDLM even achieves lower perplexity than the AR model. In addition, Uni-EDLM performs better than other energy baselines under both pre-training and fine-tuning settings. These results demonstrate the effectiveness of Uni-EDLM.

\textbf{Generative Perplexity (Gen PPL)}: The Gen PPL under different oracle Large Language Models (LLMs), along with entropy, is reported in Table~\ref{exp:tab_gen_ppl}. In this evaluation we set energy-based sampling window size $w$ as 1 to show the full power of Uni-E. Uni-EDLM consistently outperforms the baselines across all timesteps. Moreover, Uni-EDLM achieves comparable (22.8 vs 22.9 on Llama2) or better (37.5 vs 40.3 on Llama3) performance than the AR model while using only half of the timesteps (512). This indicates that Uni-E improves generation quality and reduces the number of required timesteps. Furthermore, the improvement over EDLM increases as the number of timesteps grows. A possible reason is that Inv-E reduces non-invariant error accumulation, which benefits more for longer timestep. The remask strategy can only perform well in large timesteps as it needs more steps to fix "error". As for entropy, Uni-EDLM maintains a similar level of diversity as others.

\begin{wrapfigure}{r}{0.38\textwidth} 
  \vspace{-15pt}
  \includegraphics[width=0.37\textwidth]{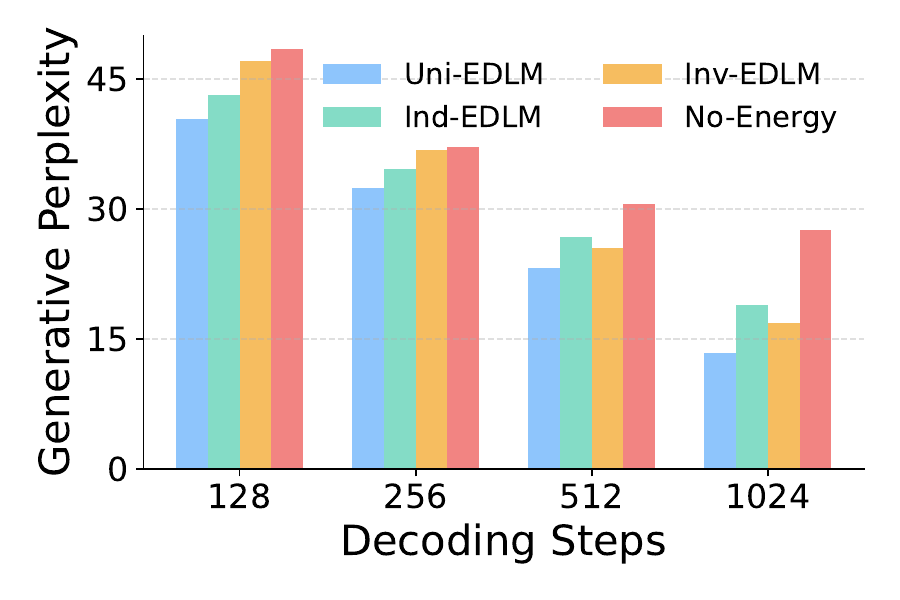} 
  \caption{Ablation studies. Inv-EDLM only uses invariant energy. Ind-EDLM only uses independent energy.}
  \label{fig:ablation}
\end{wrapfigure}
\textbf{Ablation Studies}: We use Uni-EDLM-AR as an example to conduct an ablation study to examine the influence of different energy terms. To show the full effect, we also set the window size $w$ as 1. The Gen PPL via Llama2 is shown in Figure~\ref{fig:ablation}. No energy has the same performance as the diffusion backbone MDLM. Using only the invariant energy (Inv-EDLM) provides limited benefits at small timesteps but yields significant improvements at larger timesteps. This observation suggests that invariant decoding is more beneficial at larger timesteps. In contrast, at smaller timesteps, independence becomes more important, as more tokens are decoded at each step. Using only the independent energy (Ind-EDLM) shows a larger gap compared to Uni-EDLM at larger timesteps, due to the lack of invariance modeling.
\begin{figure}[h!]
\centering
\includegraphics[width=0.95\textwidth]{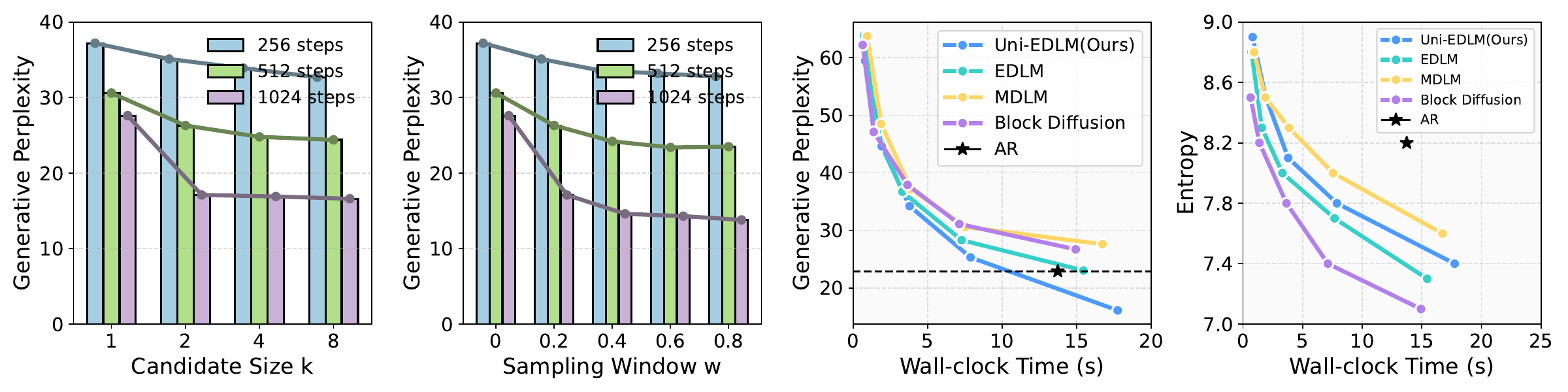}
\caption{Hyper-parameter and efficiency analysis. From left to right, the first two figures show the influence of candidate size $k$ and sampling window $w$ respectively; the last two figures show the trade-off between Gen PPL/entropy and wall-clock time. Llama2 is used as the oracle model.}
\label{exp:fig_quality_time_tradeoff}
\vspace{-12pt}
\end{figure}

\textbf{Analysis}: We conduct a hyperparameter study and efficiency analysis, with results shown in Figure~\ref{exp:fig_quality_time_tradeoff}. We use Uni-EDLM-AR as the example, where $k$ is set as 2 and $w$ is set as 0.2 by default. For the candidate size $k$, we find that extra candidate sampling mainly benefits small timesteps, with limited gains at larger timesteps. This is because when the number of timesteps is small, additional candidate sampling can improve decoding independency. For the sampling window $w$, we observe that applying energy function in the first 40\% steps can significantly improves generation quality, while the gain decreases substantially in the remaining 60\% steps (less than 3 points). This is consistent with the observation in~\cite{xu2024energy} that most errors occur in the early decoding steps. Based on this, we argue that applying Uni-E only in the first 20\% of steps with a candidate size of 2 is sufficient. As shown in the trade-off between Gen PPL and average wall-clock time, Uni-EDLM requires less time than the AR model to achieve the same performance, providing about a 22\% speedup (10.5s vs. 13.5s). In terms of entropy, all DLMs demonstrate better generative diversity than the AR model.

\textbf{Case Study}: Case study is given in Appendix~\ref{app:case_study} to show the effects of invariance and independency.

\subsection{Experiments on Diffusion Large Language Models (DLLMs)}
\label{subsec:dllm_exp}
Uni-E is model agnostic and can be easily integrated into DLLMs: for $\mathbf{x}_{t}$ we generates $k$ clean candidates $\{\mathbf{x}_{0}^{j}\}_{j=1}^{k}$ from the DLLM logits. To better align with the decoding trajectory of DLLMs, we use the highest logit of tokens as the weight to select $m$ unmasked positions for each $\mathbf{x}_{0}^{j}$, masking the remaining positions in $\mathbf{x}_{0}^{j}$ to generate $\mathbf{x}_{s}^{j}$. We calculate energy via Eq.~\ref{eq:uni_ebm} and select the $\mathbf{x}_{s}^{j}$ with lowest energy as the decoding result. The algorithm of this process are provided in Appendix~\ref{app:sampling_alg_dllm}. This integration is similar to speculative decoding~\cite{israel2025accelerating,gao2025self}, but the key difference is that \textit{existing speculative decodings are heuristic and ignore the decoding invariance.} We provide a detailed discussion in Appendix~\ref{app:uni_e_comparison_sec}.

\textbf{Settings:} Following~\cite{israel2025accelerating}, we use the pretrained Qwen2.5-0.5B~\cite{yang2025qwen3} as the proxy AR model, and LLaDa-8B-Instruct~\cite{nie2025large} and Dream-7B-Instruct~\cite{ye2025dream} as the DLLM backbones. We use these models to compute Uni-E directly. We set $k=3$ and $m=4$ for energy-based decoding, and apply energy decoding only in the first 20\% of decoding steps. We evaluate on the following benchmark datasets: MATH500~\cite{hendrycks2021measuring}, GSM8K~\cite{cobbe2021training}, MBPP~\cite{austin2021program}, LiveCodeBench V2(LCB V2)~\cite{jain2024livecodebench}, LiveBench~\cite{white2024livebench}, MMLU~\cite{hendrycks2020measuring} and HellaSwag~\cite{zellers2019hellaswag}. We compare against (i) Qwen2.5-0.5B, Qwen2.5-7B~\cite{yang2025qwen3}, and Llama3.1-8B~\cite{grattafiori2024llama}; (ii) Speculative decoding APD~\cite{israel2025accelerating}; (iii) Dependency decoding DAWN~\cite{luo2026dawn}; and (iv) Remask CORE~\cite{zhai2026core}. The detailed settings are provided in Appendix~\ref{app:dllm_exp_settings}.

\begin{table*}[h!]
\vspace{-3pt}
\centering
\footnotesize
\caption{Benchmark results. $^{\dagger}$ taken from~\cite{wang2025revolutionizing}. $^{\ddagger}$ taken from ~\cite{yang2025qwen3}. $^{*}$ means reproduced results.}
\vspace{-3pt}
\setlength{\tabcolsep}{4pt}
\begin{tabular}{l|cc|ccc|cc}
\toprule
Model 
& MATH500 & GSM8K 
& MBPP & LCB V2 & LiveBench 
& MMLU & HellaSwag \\
\midrule


Qwen2.5-7B$^{\dagger}$ 
& 74.0 & 89.9 & 74.9 & 26.9 & 31.1 & 74.2 & 80.2 \\

Qwen2.5-0.5B$^{\ddagger}$
& 33.8 & 41.6 & 39.3 & 10.4 & 12.1 & 47.5 & 52.1 \\

\midrule
LLaDa-8B$^{*}$ 
& 35.2 & 78.1 & 34.4 & 5.8 & 4.1 & 57.3 & 68.2 \\

+DAWN$^{*}$
& 38.3 (\textcolor{darkgreen}{+3.1}) 
& 79.0 (\textcolor{darkgreen}{+0.9}) 
& 33.6 (\textcolor{darkred}{-0.8}) 
& 5.5 (\textcolor{darkred}{-0.3}) 
& 4.8 (\textcolor{darkgreen}{+0.7}) 
& 58.6 (\textcolor{darkgreen}{+1.3}) 
& 68.8 (\textcolor{darkgreen}{+0.6}) \\

+APD$^{*}$
& 36.8 (\textcolor{darkgreen}{+1.6}) 
& 77.5 (\textcolor{darkred}{-0.6}) 
& 32.6 (\textcolor{darkred}{-1.8}) 
& 7.1 (\textcolor{darkgreen}{+1.3}) 
& 6.3 (\textcolor{darkgreen}{+2.2}) 
& 60.5 (\textcolor{darkgreen}{+3.2}) 
& 69.6 (\textcolor{darkgreen}{+1.4}) \\

+CORE$^{*}$
& 35.6 (\textcolor{darkgreen}{+0.4}) 
& 78.7 (\textcolor{darkgreen}{+0.6}) 
& 33.9 (\textcolor{darkred}{-0.5}) 
& 4.9 (\textcolor{darkred}{-0.9}) 
& 3.0 (\textcolor{darkred}{-1.1}) 
& 56.5 (\textcolor{darkred}{-0.8}) 
& 68.0 (\textcolor{darkred}{-0.2}) \\

\rowcolor{gray!15}
+Ind-Energy
& 37.6 (\textcolor{darkgreen}{+2.4}) 
& 79.5 (\textcolor{darkgreen}{+1.4}) 
& 34.9 (\textcolor{darkgreen}{+0.5}) 
& 9.2 (\textcolor{darkgreen}{+3.4}) 
& 8.6 (\textcolor{darkgreen}{+4.5}) 
& 62.7 (\textcolor{darkgreen}{+5.4}) 
& 69.9 (\textcolor{darkgreen}{+1.7}) \\

\rowcolor{gray!15}
+Inv-Energy
& 36.5 (\textcolor{darkgreen}{+1.3}) 
& 78.7 (\textcolor{darkgreen}{+0.6}) 
& 35.0 (\textcolor{darkgreen}{+0.6}) 
& 7.8 (\textcolor{darkgreen}{+2.0}) 
& 5.7 (\textcolor{darkgreen}{+1.6}) 
& 59.4 (\textcolor{darkgreen}{+2.1}) 
& 69.0 (\textcolor{darkgreen}{+0.8}) \\

\rowcolor{gray!15}
+Uni-Energy
& 39.2 (\textcolor{darkgreen}{+4.0}) 
& 80.1 (\textcolor{darkgreen}{+2.0}) 
& 35.2 (\textcolor{darkgreen}{+0.8}) 
& 10.5 (\textcolor{darkgreen}{+4.7}) 
& 10.9 (\textcolor{darkgreen}{+6.8}) 
& 66.5 (\textcolor{darkgreen}{+9.2}) 
& 71.2 (\textcolor{darkgreen}{+3.0}) \\

\midrule

Dream-7B$^{*}$ 
& 34.4 & 69.7 & 55.4 & 7.5 & 7.1 & 61.8 & 70.6 \\


+APD$^{*}$
& 35.9 (\textcolor{darkgreen}{+1.5}) 
& 71.4 (\textcolor{darkgreen}{+1.7}) 
& 55.8 (\textcolor{darkgreen}{+0.4}) 
& 9.0 (\textcolor{darkgreen}{+1.5}) 
& 8.2 (\textcolor{darkgreen}{+1.1}) 
& 61.3 (\textcolor{darkred}{-0.5}) 
& 71.3 (\textcolor{darkgreen}{+0.7}) \\

+CORE$^{*}$
& 34.6 (\textcolor{darkgreen}{+0.2}) 
& 70.1 (\textcolor{darkgreen}{+0.4}) 
& 56.0 (\textcolor{darkgreen}{+0.6}) 
& 6.3 (\textcolor{darkred}{-1.2}) 
& 6.5 (\textcolor{darkred}{-0.6}) 
& 62.2 (\textcolor{darkgreen}{+0.4}) 
& 69.8 (\textcolor{darkred}{-0.8}) \\


\rowcolor{gray!15}
+Ind-Energy
& 35.7 (\textcolor{darkgreen}{+1.3}) 
& 71.3 (\textcolor{darkgreen}{+1.6}) 
& 56.7 (\textcolor{darkgreen}{+1.3}) 
& 9.8 (\textcolor{darkgreen}{+2.3}) 
& 9.3 (\textcolor{darkgreen}{+2.2}) 
& 62.4 (\textcolor{darkgreen}{+0.6}) 
& 72.6 (\textcolor{darkgreen}{+2.0}) \\

\rowcolor{gray!15}
+Inv-Energy
& 35.2 (\textcolor{darkgreen}{+0.8}) 
& 70.2 (\textcolor{darkgreen}{+0.5}) 
& 56.3 (\textcolor{darkgreen}{+0.9}) 
& 8.1 (\textcolor{darkgreen}{+0.6}) 
& 8.8 (\textcolor{darkgreen}{+1.7}) 
& 62.7 (\textcolor{darkgreen}{+0.9}) 
& 71.4 (\textcolor{darkgreen}{+0.8}) \\

\rowcolor{gray!15}
+Uni-Energy
& 36.4 (\textcolor{darkgreen}{+2.0}) 
& 71.9 (\textcolor{darkgreen}{+2.2}) 
& 57.2 (\textcolor{darkgreen}{+1.8}) 
& 11.1 (\textcolor{darkgreen}{+3.6}) 
& 9.9 (\textcolor{darkgreen}{+2.8}) 
& 63.2 (\textcolor{darkgreen}{+1.4}) 
& 73.7 (\textcolor{darkgreen}{+3.1}) \\

\bottomrule
\end{tabular}
\label{tab:unie_results}
\end{table*}

\textbf{Results:} Table~\ref{tab:unie_results} presents performance across benchmark datasets and full results are in Appendix~\ref{app:dllm_bench_full}. Applying Uni-E consistently improves the performance of DLLMs, with gains exceeding those of APD and DAWN. This advantage is likely due to Uni-E jointly modeling both independency and invariance. CORE has limited performance since the remask strategy is unstable. Across MATH500, GSM8K, LCB V2, MMLU, and HellaSwag, Uni-E outperforms both the AR proxy model and the corresponding DLLM, demonstrating its effectiveness. In contrast, using only independent energy (Ind-E) or invariant energy (Inv-E) yields limited improvements, which is consistent with the observations in Figure~\ref{fig:ablation}, suggesting that modeling only dependency or invariance alone is insufficient.

\begin{wrapfigure}{r}{0.30\textwidth}
    \centering
    \includegraphics[width=0.29\textwidth]{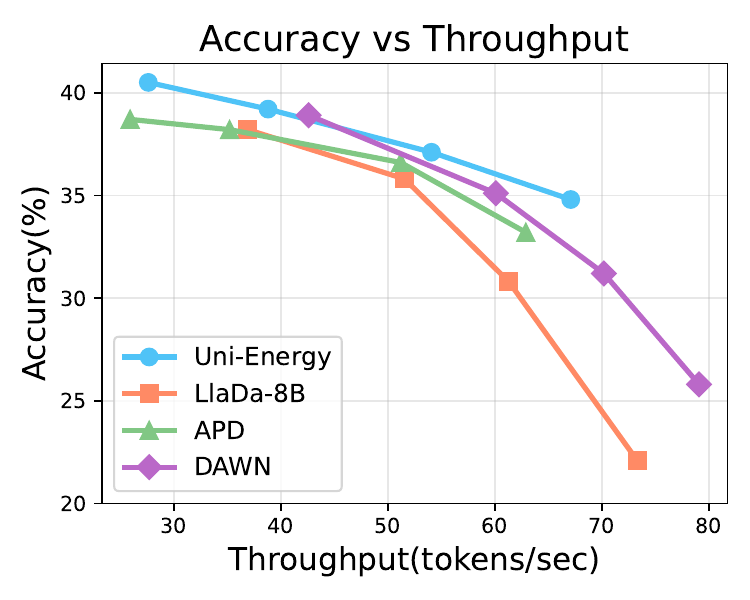}
    \caption{Efficiency Analysis}
    \label{exp:dllm_tradeoff}
\end{wrapfigure}
\textbf{Analysis:} To evaluate efficiency, we follow~\cite{israel2025accelerating} and measure throughput as the average number of tokens generated per second. We use MATH500 as an example. The results are shown in Figure~\ref{exp:dllm_tradeoff}, we observe that increasing decoding parallelism significantly degrades the performance of DLLMs. Incorporating proxy AR models such as APD or Uni-E mitigates this degradation, as they partially capture token dependencies. Notably, Uni-Energy consistently achieves better performance at the same throughput. Although it introduces extra computational cost, it improves generation quality and enables more tokens to be decoded both independently and invariantly. We provide the hyper-parameter analysis in Appendix~\ref{app:dllm_hyperparameter_study}.
\section{Conclusion}
\label{sec:conclusion}
In this paper, we first analyze the gap between Diffusion Language Models (DLMs) and the ground-truth distribution, which is governed by: (i) model capacity, (ii) dependency, and (iii) invariance. To address all these factors, we then propose an unified energy (Uni-E) that combines an invariant energy (Inv-E) and an independent energy (Ind-E). We also give theoretical analysis to show that Uni-E can be calculated efficiently and accurately, and can mitigate the distribution shift caused by dependent or non-invariant decoding. Finally, we conduct extensive experiments on both Diffusion Language Models (DLMs) and Diffusion Large Language Models (DLLMs) to demonstrate its effectiveness.

\clearpage

{
\printbibliography[heading=bibintoc,title={Bibliography}]

@article{nie2025large,
  title={Large language diffusion models},
  author={Nie, Shen and Zhu, Fengqi and You, Zebin and Zhang, Xiaolu and Ou, Jingyang and Hu, Jun and Zhou, Jun and Lin, Yankai and Wen, Ji-Rong and Li, Chongxuan},
  journal={arXiv preprint arXiv:2502.09992},
  year={2025}
}

@article{bie2025llada2,
  title={Llada2.0: Scaling up diffusion language models to 100b},
  author={Bie, Tiwei and Cao, Maosong and Chen, Kun and Du, Lun and Gong, Mingliang and Gong, Zhuochen and Gu, Yanmei and Hu, Jiaqi and Huang, Zenan and Lan, Zhenzhong and others},
  journal={arXiv preprint arXiv:2512.15745},
  year={2025}
}

@article{zhao2023survey,
  title={A survey of large language models},
  author={Zhao, Wayne Xin and Zhou, Kun and Li, Junyi and Tang, Tianyi and Wang, Xiaolei and Hou, Yupeng and Min, Yingqian and Zhang, Beichen and Zhang, Junjie and Dong, Zican and others},
  journal={arXiv preprint arXiv:2303.18223},
  volume={1},
  number={2},
  pages={1--124},
  year={2023}
}

@article{chang2024survey,
  title={A survey on evaluation of large language models},
  author={Chang, Yupeng and Wang, Xu and Wang, Jindong and Wu, Yuan and Yang, Linyi and Zhu, Kaijie and Chen, Hao and Yi, Xiaoyuan and Wang, Cunxiang and Wang, Yidong and others},
  journal={ACM transactions on intelligent systems and technology},
  volume={15},
  number={3},
  pages={1--45},
  year={2024},
  publisher={ACM New York, NY}
}

@article{li2025survey,
  title={A survey on diffusion language models},
  author={Li, Tianyi and Chen, Mingda and Guo, Bowei and Shen, Zhiqiang},
  journal={arXiv preprint arXiv:2508.10875},
  year={2025}
}

@article{zou2023survey,
  title={A survey of diffusion models in natural language processing},
  author={Zou, Hao and Kim, Zae Myung and Kang, Dongyeop},
  journal={arXiv preprint arXiv:2305.14671},
  year={2023}
}

@article{yi2024diffusion,
  title={Diffusion models in text generation: a survey},
  author={Yi, Qiuhua and Chen, Xiangfan and Zhang, Chenwei and Zhou, Zehai and Zhu, Linan and Kong, Xiangjie},
  journal={PeerJ Computer Science},
  volume={10},
  pages={e1905},
  year={2024},
  publisher={PeerJ Inc.}
}

@article{sahoo2024simple,
  title={Simple and effective masked diffusion language models},
  author={Sahoo, Subham and Arriola, Marianne and Schiff, Yair and Gokaslan, Aaron and Marroquin, Edgar and Chiu, Justin and Rush, Alexander and Kuleshov, Volodymyr},
  journal={Advances in Neural Information Processing Systems},
  volume={37},
  pages={130136--130184},
  year={2024}
}

@article{austin2021structured,
  title={Structured denoising diffusion models in discrete state-spaces},
  author={Austin, Jacob and Johnson, Daniel D and Ho, Jonathan and Tarlow, Daniel and Van Den Berg, Rianne},
  journal={Advances in neural information processing systems},
  volume={34},
  pages={17981--17993},
  year={2021}
}

@article{lou2023discrete,
  title={Discrete diffusion modeling by estimating the ratios of the data distribution},
  author={Lou, Aaron and Meng, Chenlin and Ermon, Stefano},
  journal={arXiv preprint arXiv:2310.16834},
  year={2023}
}

@article{arriola2025block,
  title={Block diffusion: Interpolating between autoregressive and diffusion language models},
  author={Arriola, Marianne and Gokaslan, Aaron and Chiu, Justin T and Yang, Zhihan and Qi, Zhixuan and Han, Jiaqi and Sahoo, Subham Sekhar and Kuleshov, Volodymyr},
  journal={arXiv preprint arXiv:2503.09573},
  year={2025}
}

@article{sahoo2025diffusion,
  title={The diffusion duality},
  author={Sahoo, Subham Sekhar and Deschenaux, Justin and Gokaslan, Aaron and Wang, Guanghan and Chiu, Justin and Kuleshov, Volodymyr},
  journal={arXiv preprint arXiv:2506.10892},
  year={2025}
}

@article{xu2024energy,
  title={Energy-based diffusion language models for text generation},
  author={Xu, Minkai and Geffner, Tomas and Kreis, Karsten and Nie, Weili and Xu, Yilun and Leskovec, Jure and Ermon, Stefano and Vahdat, Arash},
  journal={arXiv preprint arXiv:2410.21357},
  year={2024}
}

@article{liu2024discrete,
  title={Discrete copula diffusion},
  author={Liu, Anji and Broadrick, Oliver and Niepert, Mathias and Broeck, Guy Van den},
  journal={arXiv preprint arXiv:2410.01949},
  year={2024}
}

@article{li2026breaking,
  title={Breaking the Factorization Barrier in Diffusion Language Models},
  author={Li, Ian and Shao, Zilei and Wang, Benjie and Yu, Rose and Broeck, Guy Van den and Liu, Anji},
  journal={arXiv preprint arXiv:2603.00045},
  year={2026}
}

@article{luo2026dawn,
  title={DAWN: Dependency-Aware Fast Inference for Diffusion LLMs},
  author={Luo, Lizhuo and Shi, Zhuoran and Luo, Jiajun and Wang, Zhi and Ren, Shen and Wang, Wenya and Zhang, Tianwei},
  journal={arXiv preprint arXiv:2602.06953},
  year={2026}
}

@article{wu2025fast,
  title={Fast-dllm: Training-free acceleration of diffusion llm by enabling kv cache and parallel decoding},
  author={Wu, Chengyue and Zhang, Hao and Xue, Shuchen and Liu, Zhijian and Diao, Shizhe and Zhu, Ligeng and Luo, Ping and Han, Song and Xie, Enze},
  journal={arXiv preprint arXiv:2505.22618},
  year={2025}
}

@article{ye2025dream,
  title={Dream 7b: Diffusion large language models},
  author={Ye, Jiacheng and Xie, Zhihui and Zheng, Lin and Gao, Jiahui and Wu, Zirui and Jiang, Xin and Li, Zhenguo and Kong, Lingpeng},
  journal={arXiv preprint arXiv:2508.15487},
  year={2025}
}

@inproceedings{zhou2022model,
  title={Model agnostic sample reweighting for out-of-distribution learning},
  author={Zhou, Xiao and Lin, Yong and Pi, Renjie and Zhang, Weizhong and Xu, Renzhe and Cui, Peng and Zhang, Tong},
  booktitle={International conference on machine learning},
  pages={27203--27221},
  year={2022},
  organization={PMLR}
}

@article{lecun2006tutorial,
  title={A tutorial on energy-based learning},
  author={LeCun, Yann and Chopra, Sumit and Hadsell, Raia and Ranzato, M and Huang, Fujie and others},
  journal={Predicting structured data},
  volume={1},
  number={0},
  year={2006}
}

@inproceedings{haarnoja2017reinforcement,
  title={Reinforcement learning with deep energy-based policies},
  author={Haarnoja, Tuomas and Tang, Haoran and Abbeel, Pieter and Levine, Sergey},
  booktitle={International conference on machine learning},
  pages={1352--1361},
  year={2017},
  organization={PMLR}
}

@article{lavenant2025error,
  title={Error Bounds and Optimal Schedules for Masked Diffusions with Factorized Approximations},
  author={Lavenant, Hugo and Zanella, Giacomo},
  journal={arXiv preprint arXiv:2510.25544},
  year={2025}
}

@article{chen2025optimal,
  title={Optimal inference schedules for masked diffusion models},
  author={Chen, Sitan and Cong, Kevin and Li, Jerry},
  journal={arXiv preprint arXiv:2511.04647},
  year={2025}
}

@inproceedings{gutmann2010noise,
  title={Noise-contrastive estimation: A new estimation principle for unnormalized statistical models},
  author={Gutmann, Michael and Hyv{\"a}rinen, Aapo},
  booktitle={Proceedings of the thirteenth international conference on artificial intelligence and statistics},
  pages={297--304},
  year={2010},
  organization={JMLR Workshop and Conference Proceedings}
}

@book{hammersley2013monte,
  title={Monte carlo methods},
  author={Hammersley, John},
  year={2013},
  publisher={Springer Science \& Business Media}
}

@article{james1980monte,
  title={Monte Carlo theory and practice},
  author={James, Frederick},
  journal={Reports on progress in Physics},
  volume={43},
  number={9},
  pages={1145--1189},
  year={1980}
}

@misc{Gokaslan2019OpenWeb,
	title={OpenWebText Corpus},
	author={Aaron Gokaslan and Vanya Cohen},
	howpublished={\url{http://Skylion007.github.io/OpenWebTextCorpus}},
	year={2019}
}

@article{marcus1993building,
  title={Building a large annotated corpus of English: The Penn Treebank},
  author={Marcus, Mitch and Santorini, Beatrice and Marcinkiewicz, Mary Ann},
  journal={Computational linguistics},
  volume={19},
  number={2},
  pages={313--330},
  year={1993}
}

@article{merity2016pointer,
  title={Pointer sentinel mixture models},
  author={Merity, Stephen and Xiong, Caiming and Bradbury, James and Socher, Richard},
  journal={arXiv preprint arXiv:1609.07843},
  year={2016}
}

@inproceedings{paperno2016lambada,
  title={The LAMBADA dataset: Word prediction requiring a broad discourse context},
  author={Paperno, Denis and Kruszewski, Germ{\'a}n and Lazaridou, Angeliki and Pham, Ngoc-Quan and Bernardi, Raffaella and Pezzelle, Sandro and Baroni, Marco and Boleda, Gemma and Fern{\'a}ndez, Raquel},
  booktitle={Proceedings of the 54th annual meeting of the association for computational linguistics (volume 1: Long papers)},
  pages={1525--1534},
  year={2016}
}

@article{chelba2013one,
  title={One billion word benchmark for measuring progress in statistical language modeling},
  author={Chelba, Ciprian and Mikolov, Tomas and Schuster, Mike and Ge, Qi and Brants, Thorsten and Koehn, Phillipp and Robinson, Tony},
  journal={arXiv preprint arXiv:1312.3005},
  year={2013}
}

@article{zhang2015character,
  title={Character-level convolutional networks for text classification},
  author={Zhang, Xiang and Zhao, Junbo and LeCun, Yann},
  journal={Advances in neural information processing systems},
  volume={28},
  year={2015}
}

@inproceedings{cohan2018discourse,
  title={A discourse-aware attention model for abstractive summarization of long documents},
  author={Cohan, Arman and Dernoncourt, Franck and Kim, Doo Soon and Bui, Trung and Kim, Seokhwan and Chang, Walter and Goharian, Nazli},
  booktitle={Proceedings of the 2018 Conference of the North American Chapter of the Association for Computational Linguistics: Human Language Technologies, Volume 2 (Short Papers)},
  pages={615--621},
  year={2018}
}

@article{vaswani2017attention,
  title={Attention is all you need},
  author={Vaswani, Ashish and Shazeer, Noam and Parmar, Niki and Uszkoreit, Jakob and Jones, Llion and Gomez, Aidan N and Kaiser, {\L}ukasz and Polosukhin, Illia},
  journal={Advances in neural information processing systems},
  volume={30},
  year={2017}
}

@article{radford2019language,
  title={Language models are unsupervised multitask learners},
  author={Radford, Alec and Wu, Jeffrey and Child, Rewon and Luan, David and Amodei, Dario and Sutskever, Ilya and others},
  journal={OpenAI blog},
  volume={1},
  number={8},
  pages={9},
  year={2019}
}

@article{israel2025accelerating,
  title={Accelerating diffusion llms via adaptive parallel decoding},
  author={Israel, Daniel and Broeck, Guy Van den and Grover, Aditya},
  journal={arXiv preprint arXiv:2506.00413},
  year={2025}
}

@article{gao2025self,
  title={Self speculative decoding for diffusion large language models},
  author={Gao, Yifeng and Ji, Ziang and Wang, Yuxuan and Qi, Biqing and Xu, Hanlin and Zhang, Linfeng},
  journal={arXiv preprint arXiv:2510.04147},
  year={2025}
}

@article{yang2025qwen3,
  title={Qwen3 technical report},
  author={Yang, An and Li, Anfeng and Yang, Baosong and Zhang, Beichen and Hui, Binyuan and Zheng, Bo and Yu, Bowen and Gao, Chang and Huang, Chengen and Lv, Chenxu and others},
  journal={arXiv preprint arXiv:2505.09388},
  year={2025}
}

@article{cobbe2021training,
  title={Training verifiers to solve math word problems},
  author={Cobbe, Karl and Kosaraju, Vineet and Bavarian, Mohammad and Chen, Mark and Jun, Heewoo and Kaiser, Lukasz and Plappert, Matthias and Tworek, Jerry and Hilton, Jacob and Nakano, Reiichiro and others},
  journal={arXiv preprint arXiv:2110.14168},
  year={2021}
}

@article{hendrycks2021measuring,
  title={Measuring mathematical problem solving with the math dataset},
  author={Hendrycks, Dan and Burns, Collin and Kadavath, Saurav and Arora, Akul and Basart, Steven and Tang, Eric and Song, Dawn and Steinhardt, Jacob},
  journal={arXiv preprint arXiv:2103.03874},
  year={2021}
}

@article{jain2024livecodebench,
  title={Livecodebench: Holistic and contamination free evaluation of large language models for code},
  author={Jain, Naman and Han, King and Gu, Alex and Li, Wen-Ding and Yan, Fanjia and Zhang, Tianjun and Wang, Sida and Solar-Lezama, Armando and Sen, Koushik and Stoica, Ion},
  journal={arXiv preprint arXiv:2403.07974},
  year={2024}
}

@article{white2024livebench,
  title={Livebench: A challenging, contamination-free llm benchmark},
  author={White, Colin and Dooley, Samuel and Roberts, Manley and Pal, Arka and Feuer, Ben and Jain, Siddhartha and Shwartz-Ziv, Ravid and Jain, Neel and Saifullah, Khalid and Naidu, Siddartha and others},
  journal={arXiv preprint arXiv:2406.19314},
  volume={4},
  pages={2},
  year={2024}
}

@article{hendrycks2020measuring,
  title={Measuring massive multitask language understanding},
  author={Hendrycks, Dan and Burns, Collin and Basart, Steven and Zou, Andy and Mazeika, Mantas and Song, Dawn and Steinhardt, Jacob},
  journal={arXiv preprint arXiv:2009.03300},
  year={2020}
}

@inproceedings{zellers2019hellaswag,
  title={Hellaswag: Can a machine really finish your sentence?},
  author={Zellers, Rowan and Holtzman, Ari and Bisk, Yonatan and Farhadi, Ali and Choi, Yejin},
  booktitle={Proceedings of the 57th annual meeting of the association for computational linguistics},
  pages={4791--4800},
  year={2019}
}

@article{austin2021program,
  title={Program synthesis with large language models},
  author={Austin, Jacob and Odena, Augustus and Nye, Maxwell and Bosma, Maarten and Michalewski, Henryk and Dohan, David and Jiang, Ellen and Cai, Carrie and Terry, Michael and Le, Quoc and others},
  journal={arXiv preprint arXiv:2108.07732},
  year={2021}
}

@article{grattafiori2024llama,
  title={The llama 3 herd of models},
  author={Grattafiori, Aaron and Dubey, Abhimanyu and Jauhri, Abhinav and Pandey, Abhinav and Kadian, Abhishek and Al-Dahle, Ahmad and Letman, Aiesha and Mathur, Akhil and Schelten, Alan and Vaughan, Alex and others},
  journal={arXiv preprint arXiv:2407.21783},
  year={2024}
}

@article{wang2025revolutionizing,
  title={Revolutionizing reinforcement learning framework for diffusion large language models},
  author={Wang, Yinjie and Yang, Ling and Li, Bowen and Tian, Ye and Shen, Ke and Wang, Mengdi},
  journal={arXiv preprint arXiv:2509.06949},
  year={2025}
}

@article{hoogeboom2021autoregressive,
  title={Autoregressive diffusion models},
  author={Hoogeboom, Emiel and Gritsenko, Alexey A and Bastings, Jasmijn and Poole, Ben and Berg, Rianne van den and Salimans, Tim},
  journal={arXiv preprint arXiv:2110.02037},
  year={2021}
}

@article{campbell2022continuous,
  title={A continuous time framework for discrete denoising models},
  author={Campbell, Andrew and Benton, Joe and De Bortoli, Valentin and Rainforth, Thomas and Deligiannidis, George and Doucet, Arnaud},
  journal={Advances in Neural Information Processing Systems},
  volume={35},
  pages={28266--28279},
  year={2022}
}

@article{wang2025remasking,
  title={Remasking discrete diffusion models with inference-time scaling},
  author={Wang, Guanghan and Schiff, Yair and Sahoo, Subham Sekhar and Kuleshov, Volodymyr},
  journal={arXiv preprint arXiv:2503.00307},
  year={2025}
}

@article{yao2026remask,
  title={Remask, Don't Replace: Token-to-Mask Refinement in Masked Diffusion Language Models},
  author={Yao, Lin},
  journal={arXiv preprint arXiv:2604.18738},
  year={2026}
}

@article{zhai2026core,
  title={CoRe: Context-Robust Remasking for Diffusion Language Models},
  author={Zhai, Kevin and Mollah, Sabbir and Wang, Zhenyi and Shah, Mubarak},
  journal={arXiv preprint arXiv:2602.04096},
  year={2026}
}
}

\clearpage

\section{Appendix}

\subsection{Training Algorithm}
\label{app:training_alg}

We provide the NCE training algorithm for our Uni-EDLM in Algorithm~\ref{alg:training}.
\begin{algorithm}[H]
\caption{Training Uni-EDLM with Noise Contrastive Estimation (NCE)}
\label{alg:training}
\begin{algorithmic}[1]
\REQUIRE Training dataset $\mathcal{D}$, AR model $p_{\text{AR}}$, diffusion model $p_\theta$, learning rate $\eta$
\STATE Freeze parameters of $p_{\text{AR}}$
\WHILE{not converged}
    \STATE Sample clean data $\mathbf{x}_0 \sim \mathcal{D}$ and diffusion timestep $t \sim \mathcal{U}(0,1)$
    \STATE Sample $\mathbf{x}_t \sim q(\mathbf{x}_t | \mathbf{x}_0)$
    \STATE Sample $s$ with $s < t$ and $\mathbf{x}_s \sim q(\mathbf{x}_s | \mathbf{x}_t, \mathbf{x}_0)$
    \STATE Compute energy:
    \[
    E_\phi(\mathbf{x}_0, \mathbf{x}_s, \mathbf{x}_t) = \log \frac{p_\theta(\mathbf{x}_s|\mathbf{x}_t)}{p_{\text{AR}}(\mathbf{x}_s|\mathbf{x}_t)} + \log \frac{p_\theta(\mathbf{x}_0|\mathbf{x}_t)}{p_{\text{AR}}(\mathbf{x}_0|\mathbf{x}_s)}
    \]
    
    \STATE Sample $\hat{\mathbf{x}}_0 \sim p_\theta(\mathbf{x}_0 | \mathbf{x}_t)$
    \STATE Sample $\hat{\mathbf{x}}_s \sim q(\mathbf{x}_s | \mathbf{x}_t, \hat{\mathbf{x}}_0)$
    \STATE Compute energy:
    \[
    E_\phi(\hat{\mathbf{x}}_0, \hat{\mathbf{x}}_s, \mathbf{x}_t) = \log \frac{p_\theta(\hat{\mathbf{x}}_s|\mathbf{x}_t)}{p_{\text{AR}}(\hat{\mathbf{x}}_s|\mathbf{x}_t)} + \log \frac{p_\theta(\hat{\mathbf{x}}_0|\mathbf{x}_t)}{p_{\text{AR}}(\hat{\mathbf{x}}_0|\mathbf{x}_s)}
    \]
    
    \STATE Compute loss:
    \[
    \mathcal{L}_{\text{NCE}} =
    \log \left(1 + \exp(E_\phi(\mathbf{x}_0, \mathbf{x}_s, \mathbf{x}_t)) \right)
    + \log \left(1 + \exp(-E_\phi(\hat{\mathbf{x}}_0, \hat{\mathbf{x}}_s, \mathbf{x}_t)) \right)
    \]
    
    \STATE Update $\theta \leftarrow \theta - \eta \nabla_\theta \mathcal{L}_{\text{NCE}}$
\ENDWHILE
\RETURN trained parameters $\theta$
\end{algorithmic}
\label{alg:train}
\end{algorithm}

\subsection{Inference Algorithm}
\label{app:sampling_alg}
We provide the inference of Uni-EDLM in Algorithm~\ref{alg:inference}.

\begin{algorithm}[H]
\caption{Uni-EDLM Inference with Importance Sampling}
\label{alg:inference}
\begin{algorithmic}[1]
\REQUIRE Full masked sequence $\mathbf{x}_{\mathcal{T}_{1}}$, diffusion timesteps $\mathcal{T}_{1}>\mathcal{T}_{2}>...>\mathcal{T}_{T}$, candidate size $k$, window size $w$, proxy AR model $p_{\text{AR}}$, diffusion model $p_\theta$
\FOR{$t = \mathcal{T}_{1},\mathcal{T}_{2}, \dots, \mathcal{T}_{T}$}
    
    \IF{$t \in (1-w, 1]$}        
        \STATE Sample $\{\mathbf{x}_0^{1},..,\mathbf{x}_0^{k} \} \sim p_\theta(\mathbf{x}_0 | \mathbf{x}_{t})$ and sample timestep $s < t$
        \STATE Sample $\{\mathbf{x}_s^{1},..,\mathbf{x}_s^{k} \} \sim q(\mathbf{x}_s | \mathbf{x}_t, \mathbf{x}_0^{(j)})$ with $j\in\{1,2,..,k\}$
        \STATE Compute energy $E^{(j)} = E_\phi(\mathbf{x}_0^{(j)}, \mathbf{x}_s^{(j)}, \mathbf{x}_t)$ for each $(\mathbf{x}_0^{(j)}, \mathbf{x}_s^{(j)})$ pair
        \STATE $j^* = \arg\min_j E^{(j)}$
        \STATE $\tilde{\mathbf{x}}_0 = \mathbf{x}_0^{(j^*)}$
    \ELSE
        \STATE Sample $\tilde{\mathbf{x}}_0 \sim p_\theta(\mathbf{x}_0 | \mathbf{x}_t)$
    \ENDIF
    \STATE Sample $\mathbf{x}_{t-1} \sim q(\mathbf{x}_{t-1} | \mathbf{x}_t, \tilde{\mathbf{x}}_0)$
    
\ENDFOR
\RETURN $\mathbf{x}_0$
\end{algorithmic}
\label{alg:inference}
\end{algorithm}

\subsection{Inference Algorithm for Diffusion Large Language Models (DLLMs)}
\label{app:sampling_alg_dllm}
We provide the inference algorithm when integrating with Diffusion Large Language Models (DLLMs) in Algorithm~\ref{alg:inference_dllm}.
\begin{algorithm}[H]
\caption{Energy-Guided Inference for Diffusion Large Language Models (DLLMs)}
\label{alg:dllm_energy_sampling}
\begin{algorithmic}[1]
\REQUIRE Full masked sequence $\mathbf{x}_{1}$ with max length $L$, candidate size $k$, number of unmask positions $m$ for each step, window size $w$, DLLM model $p_\theta$, AR model $p_{\text{AR}}$
\FOR{$t=1, 1-\frac{m}{L}, 1-\frac{2m}{L}, ..., 0$}
\STATE Obtain token logits from DLLM: $p_\theta(\mathbf{x}_0|\mathbf{x}_t)$
\IF{$t \in (1-w, 1]$}
\FOR{$j = 1$ to $k$}
    \STATE Sample clean candidate: $\mathbf{x}_0^{(j)} \sim p_\theta(\mathbf{x}_0|\mathbf{x}_t)$
    \STATE Compute token-level logits $\{\ell_i^{(j)}\}_{i=1}^L$ for $\mathbf{x}_0^{(j)}$
    \STATE Select index set $\mathcal{I}^{(j)}$ corresponding to top-$m$ logits
    
    \STATE Generate $\mathbf{x}_s^{(j)}$:
    \[
    \mathbf{x}_{s,i}^{(j)} =
    \begin{cases}
        \mathbf{x}_{0,i}^{(j)}, & i \in \mathcal{I}^{(j)} \\
        \text{[MASK]}, & \text{otherwise}
    \end{cases}
    \]
    
    \STATE Compute energy for each $(\mathbf{x}_0^{(j)}, \mathbf{x}_s^{(j)})$ pair
    \[
    E^{(j)} =
    \log \frac{p_\theta(\mathbf{x}_s^{(j)} | \mathbf{x}_t)}{p_{\text{AR}}(\mathbf{x}_s^{(j)} | \mathbf{x}_t)}
    + \log \frac{p_\theta(\mathbf{x}_0^{(j)} | \mathbf{x}_t)}{p_{\text{AR}}(\mathbf{x}_0^{(j)} | \mathbf{x}_s^{(j)})}
    \]
\ENDFOR
\STATE $j^* = \arg\min_j E^{(j)}$
\STATE $\mathbf{x}_{t} = \mathbf{x}_s^{(j^*)}$

\ELSE
\STATE Unmask $m$ tokens with highest logits within $\mathbf{x}_{t}$ from $p_\theta(\mathbf{x}_0|\mathbf{x}_t)$

\ENDIF
\ENDFOR

\RETURN $\mathbf{x}_0$
\end{algorithmic}
\label{alg:inference_dllm}
\end{algorithm}

\subsection{Proof of Invariance Divergence Estimation}
\label{app:lemma_proof}
Here we give a formal proof to the Lemma~\ref{theory:lemma_1} as follows:
\begin{proof}
According to the invariance defination, we have:
\begin{equation}
\begin{aligned}
    \pi(\mathbf{x}_{0}^{i} | \mathbf{x}_{t+1}, \mathbf{x}_{0}^{\mathcal{V}^{-}_{i}}, \mathbf{x}_{0}^{Z_{<t}}-\mathbf{x}_{0}^{\mathcal{V}^{-}_{i}}) = \pi(\mathbf{x}_{0}^{i} | \mathbf{x}_{t+1}, \mathbf{x}_{0}^{\mathcal{V}^{-}_{i}})
\end{aligned}
\end{equation}
Then we have:
\begin{equation}
\begin{aligned}
    \frac{\pi(\mathbf{x}_{0}^{i}, \mathbf{x}_{0}^{Z_{<t}}-\mathbf{x}_{0}^{\mathcal{V}^{-}_{i}} | \mathbf{x}_{t+1}, \mathbf{x}_{0}^{\mathcal{V}^{-}_{i}})}{\pi(\mathbf{x}_{0}^{Z_{<t}}-\mathbf{x}_{0}^{\mathcal{V}^{-}_{i}} | \mathbf{x}_{t+1}, \mathbf{x}_{0}^{\mathcal{V}^{-}_{i}})} = \pi(\mathbf{x}_{0}^{i} | \mathbf{x}_{t+1}, \mathbf{x}_{0}^{\mathcal{V}^{-}_{i}})
\end{aligned}
\end{equation}
Therefore we have:
\begin{equation}
\begin{aligned}
    \pi(\mathbf{x}_{0}^{Z_{<t}}-\mathbf{x}_{0}^{\mathcal{V}^{-}_{i}} | \mathbf{x}_{0}^{i}, \mathbf{x}_{t+1}, \mathbf{x}_{0}^{\mathcal{V}^{-}_{i}}) = \pi(\mathbf{x}_{0}^{Z_{<t}}-\mathbf{x}_{0}^{\mathcal{V}^{-}_{i}} | \mathbf{x}_{t+1}, \mathbf{x}_{0}^{\mathcal{V}^{-}_{i}})
\end{aligned}
\end{equation}
The intuitive understanding is that the $i$-th token has no influence on the rest tokens beyond invariant ones. Given the invriance part defination: $\pi(\textbf{x}^{i}_{0} | \textbf{x}_{0} - \textbf{x}_{0}^{i}) = \pi(\textbf{x}_{0}^{i} | \textbf{x}_{0}^{\mathcal{V}_{i}}) = \pi(\textbf{x}_{0}^{i} | \textbf{x}_{t+1}, \textbf{x}_{0}^{\mathcal{V}^{-}_{i}})$, we know that $\pi(\textbf{x}_{0}^{\mathcal{V}^{-}_{i}}|\textbf{x}_{0}^{i}, \textbf{x}_{t+1}, \mathbf{x}_{0}^{Z_{<t}}-\mathbf{x}_{0}^{\mathcal{V}^{-}_{i}}) = \pi(\textbf{x}_{0}^{\mathcal{V}^{-}_{i}}|\textbf{x}_{0}^{i}, \textbf{x}_{t+1})$, this means the invariant tokens $\mathbf{x}_{0}^{\mathcal{V}^{-}_{i}}$ of $\mathbf{x}_{0}^{i}$ is also invariant to $\mathbf{x}_{0}^{Z_{<t}}-\mathbf{x}_{0}^{\mathcal{V}^{-}_{i}}$. This indicate a key insight that \textbf{the decoding of $\textbf{x}_{0}^{i}$ will only influence the invariant tokens and this influence will not passed to the non-invariant parts}.

\begin{table}[h!]
\center
\begin{tabular}{|c|c|c|c|c|c|c|c|c|}
\hline
                & $T_{i}^{1}$ & $T_{i}^{2}$ & ... & $T_{i}^{x}$ & $T_{n}^{1}$   & $T_{n}^{2}$   & ... & $T_{n}^{y}$   \\ \hline
Before Decoding & $p_{1}$ & $p_{2}$ & ...   & $p_{x}$ & $g_{1}$ & $g_{2}$ & ... & $g_{y}$ \\ \hline
After Decoding  & $q_{1}$ & $q_{2}$ & ...   & $q_{x}$ & $g_{1}$ & $g_{2}$ & ... & $g_{y}$ \\ \hline
After Sampling  & $q_{1}$ & $q_{2}$ & ...   & $p_{x}$ & $g_{1}$ & $g_{2}$ & ... & $g_{y}$ \\ \hline
\end{tabular}
\caption{The distribution shift before and after decoding token $\textbf{x}_{0}^{i}$}
\label{appedix:table_1}
\end{table}

Assuming in time step $t$, we decode a token $\textbf{x}_{0}^{i}$, the rest part of $\textbf{x}_{0}$ distribution before and after decoding is shown as the first two row in Table~\ref{appedix:table_1}, where $T_{i}^{j}$ denotes the $j$-th invariant token and $T_{n}^{k}$ denote the $k$-th non-invariant token. The decoding has no influence on the non-invariant tokens while the invariant one will have distribution shift from $p_{i}$ to $q_{i}$. We can easily conclude that:

\begin{equation}
\begin{aligned}
    log \frac{\pi(\mathbf{x}_{0}^{\mathcal{V}^{-}_{i}}|\mathbf{x}_{t+1})}{\pi(\mathbf{x}_{0}^{\mathcal{V}^{-}_{i}}|\mathbf{x}_{i}, \mathbf{x}_{t+1})} = log \frac{\pi(\mathbf{x}_{0}^{Z_{<t}}|\mathbf{x}_{t+1})}{\pi(\mathbf{x}_{0}^{Z_{<t}}|\mathbf{x}_{i}, \mathbf{x}_{t+1})} = \sum_{i}^{x} KL(p_{i} || q_{i}) + \sum_{i}^{y} KL(g_{i} || g_{i}) = \sum_{i}^{x} KL(p_{i} || q_{i})
\end{aligned}
\end{equation}
Therefore if we sample all the $Z_{<t}$ we will have the full esitmation of invariant part distribution. Next we will prove the estimation when sampling. Here we sample $r$ tokens from $Z_{<t}$ and get index $Z_{s}$, and after sampling the distribution is shown as the third line in Table~\ref{appedix:table_1}. Suppose we sample the first $m$ invariant and the rest $n$ tokens are non-invariant($m+n=r$), the first $m$ distribution will be influenced by $\textbf{x}_{0}^{i}$ when calculating estimation, which is $q_{i}$($i \leq m$) and the remains are the same as the non-decoding $\textbf{x}_{0}^{i}$(The "Before Decoding" line in Table~\ref{appedix:table_1}). The sampling estimation is:

\begin{equation}
\begin{aligned}
    log \frac{\pi(\mathbf{x}_{0}^{Z_{s}}|\mathbf{x}_{t+1})}{\pi(\mathbf{x}_{0}^{Z_{s}}|\mathbf{x}_{i}, \mathbf{x}_{t+1})} = \sum_{i=1}^{m} KL(p_{i} || q_{i}) + \sum_{i=m+1}^{x} KL(p_{i} || p_{i}) + \sum_{i=1}^{y} KL(g_{i} || g_{i}) = \sum_{i}^{m} KL(p_{i} || q_{i})
\end{aligned}
\end{equation}

Therefore, we have:
\begin{equation}
\begin{aligned}
    log \frac{\pi(\mathbf{x}_{0}^{\mathcal{V}^{-}_{i}}|\mathbf{x}_{t+1})}{\pi(\mathbf{x}_{0}^{\mathcal{V}^{-}_{i}}|\mathbf{x}_{i}, \mathbf{x}_{t+1})} - log \frac{\pi(\mathbf{x}_{0}^{Z_{s}}|\mathbf{x}_{t+1})}{\pi(\mathbf{x}_{0}^{Z_{s}}|\mathbf{x}_{i}, \mathbf{x}_{t+1})} = \sum_{i=m+1}^{x} KL(p_{i} || q_{i})
\end{aligned}
\label{eq:app_lemma_1}
\end{equation}

Considering the Equation~\ref{eq:app_lemma_1} just describe one specific situation that we select the first $m$-th invariant tokens, when we traverse all the possible sampling of $Z_{s}$ and calculate the average, it can be deduced easily that the possibility of picking each token is $\frac{r}{x+y}$, leading to the average residual term as:
\begin{equation}
\begin{aligned}
    \mathbb{E}_{Z_{s}} \left\{ log \frac{\pi(\mathbf{x}_{0}^{\mathcal{V}^{-}_{i}}|\mathbf{x}_{t+1})}{\pi(\mathbf{x}_{0}^{\mathcal{V}^{-}_{i}}|\mathbf{x}_{i}, \mathbf{x}_{t+1})} - log \frac{\pi(\mathbf{x}_{0}^{Z_{s}}|\mathbf{x}_{t+1})}{\pi(\mathbf{x}_{0}^{Z_{s}}|\mathbf{x}_{i}, \mathbf{x}_{t+1})} \right\} = (1-\frac{r}{x+y})\sum_{i=1}^{x} KL(p_{i} || q_{i})
\end{aligned}
\end{equation}
Therefore we have:
\begin{equation}
\begin{aligned}
    \mathbb{E}_{Z_{s}} \left\{ log \frac{\pi(\mathbf{x}_{0}^{Z_{s}}|\mathbf{x}_{t+1})}{\pi(\mathbf{x}_{0}^{Z_{s}}|\mathbf{x}_{i}, \mathbf{x}_{t+1})} \right\} = (\frac{r}{x+y})\sum_{i=1}^{x} KL(p_{i} || q_{i}) = (\frac{r}{x+y}) log \frac{\pi(\mathbf{x}_{0}^{\mathcal{V}^{-}_{i}}|\mathbf{x}_{t+1})}{\pi(\mathbf{x}_{0}^{\mathcal{V}^{-}_{i}}|\mathbf{x}_{i}, \mathbf{x}_{t+1})}
\end{aligned}
\end{equation}
So if we use $p_{AR}(.)$ as the proxy, we finally get:
\begin{equation}
\begin{aligned}
    \mathbb{E}_{Z_{s}} \left\{ log \frac{p_{AR}(\mathbf{x}_{0}^{Z_{s}}|\mathbf{x}_{t+1})}{p_{AR}(\mathbf{x}_{0}^{Z_{s}}|\mathbf{x}_{i}, \mathbf{x}_{t+1})} \right\} = (\frac{r}{x+y}) log \frac{\pi(\mathbf{x}_{0}^{\mathcal{V}^{-}_{i}}|\mathbf{x}_{t+1})}{\pi(\mathbf{x}_{0}^{\mathcal{V}^{-}_{i}}|\mathbf{x}_{i}, \mathbf{x}_{t+1})} + \gamma
\end{aligned}
\end{equation}
Therefore, we prove Lemma~\ref{theory:lemma_1}.

\end{proof}

\subsection{Proof of Uni-EDLM's Effectiveness}
\label{app:lemma_2_proof}
Here we give a formal proof to the Lemma~\ref{theory:lemma_2} as follows:
\begin{proof}
Considering all the possible decoding $s$ and $t$ pairs on clean data $\mathbf{x}_{0}$, the average KL Divergence of diffusion language model is defined as:
\begin{equation}
\begin{aligned}
    KL(\pi||p_{\theta}) = \mathbb{E}_{\mathbf{x}_{t}} \left\{\sum_{s(s<t)} \Omega(\mathbf{x}_{s}|\mathbf{x}_{t}) log \frac{\pi(\mathbf{x}_{s}|\mathbf{x}_{t})}{p_{\theta}(\mathbf{x}_{s}|\mathbf{x}_{t})} \right\},
\end{aligned}
\end{equation}
where the $\Omega(\mathbf{x}_{s}|\mathbf{x}_{t})$ is the ground truth decoding planner distribution describing the true possibility of decoding position in $\mathbf{x}_{s}$ based on $\mathbf{x}_{t}$, the average KL Divergence for Uni-EDLM is:
\begin{equation}
\begin{aligned}
    KL(\pi||p_{\phi,\theta}) = \mathbb{E}_{\mathbf{x}_{t}} \left\{\sum_{s(s<t)} \Omega(\mathbf{x}_{s}|\mathbf{x}_{t}) log \frac{\pi(\mathbf{x}_{s}|\mathbf{x}_{t})}{p_{\phi, \theta}(\mathbf{x}_{s}|\mathbf{x}_{t})} \right\},
\end{aligned}
\end{equation}
Now giving specific $t$, we consider the situation that the decoding positions do not fulfill the independence or invariance. Supposing such decodings in descendent violation degree are $Z_{1}$, $Z_{2}$, ..., $Z_{n}$($Z_{n}$ is the most "independent \& invariant" one and $Z_{1}$ is the most "non-independent\&non-invariant" one), we have:
\begin{equation}
\begin{aligned}
    \Omega(\mathbf{x}_{s}^{Z_{1}}|\mathbf{x}_{t}) < \Omega(\mathbf{x}_{s}^{Z_{2}}|\mathbf{x}_{t}) < ... < \Omega(\mathbf{x}_{s}^{Z_{n}}|\mathbf{x}_{t}),
\end{aligned}
\end{equation}
since $\Omega(.)$ represents the real possiblity of decoding order. We also assume that the diffusion language model unfortunately can not distinguish them, where:
\begin{equation}
\begin{aligned}
    p_{\theta}(\mathbf{x}_{s}^{Z_{1}}|\mathbf{x}_{t}) \approx p_{\theta}(\mathbf{x}_{s}^{Z_{2}}|\mathbf{x}_{t}) \approx ... \approx p_{\theta}(\mathbf{x}_{s}^{Z_{n}}|\mathbf{x}_{t}),
\end{aligned}
\end{equation}
According to the Uni-Energy, it will allocate lower energy to the independent and invariant decodings while allocate higher energy to the non-independent or non-invariant decodings, thus:
\begin{equation}
\begin{aligned}
    E_{uni}(\mathbf{x}_{0}, \mathbf{x}_{s}^{Z_{1}},\mathbf{x}_{t}) > E_{uni}(\mathbf{x}_{0}, \mathbf{x}_{s}^{Z_{2}}, \mathbf{x}_{t}) >... > E_{uni}(\mathbf{x}_{0}, \mathbf{x}_{s}^{Z_{n}}, \mathbf{x}_{t}),
\end{aligned}
\end{equation}
Thus we have:
\begin{equation}
\begin{aligned}
    & \frac{exp(-E_{uni}(\mathbf{x}_{0}, \mathbf{x}_{s}^{Z_{1}},\mathbf{x}_{t}))}{F(\mathbf{x}_{t})}p_{\theta}(\mathbf{x}_{s}^{Z_{1}}|\mathbf{x}_{t}) < \frac{exp(-E_{uni}(\mathbf{x}_{0}, \mathbf{x}_{s}^{Z_{2}},\mathbf{x}_{t}))}{F(\mathbf{x}_{t})}p_{\theta}(\mathbf{x}_{s}^{Z_{2}}|\mathbf{x}_{t}) \\
    & <... < \frac{exp(-E_{uni}(\mathbf{x}_{0}, \mathbf{x}_{s}^{Z_{n}},\mathbf{x}_{t}))}{F(\mathbf{x}_{t})}p_{\theta}(\mathbf{x}_{s}^{Z_{n}}|\mathbf{x}_{t}),
\end{aligned}
\end{equation}
We also know that all the possibility are normalized, which means:
\begin{equation}
\begin{aligned}
    \sum_{s,i} \frac{exp(-E_{uni}(\mathbf{x}_{0}, \mathbf{x}_{s}^{Z_{i}},\mathbf{x}_{t}))}{F(\mathbf{x}_{t})}p_{\theta}(\mathbf{x}_{s}^{Z_{i}}|\mathbf{x}_{t}) = 1, 
\end{aligned}
\end{equation}
and
\begin{equation}
\begin{aligned}
    \sum_{s,i} p_{\theta}(\mathbf{x}_{s}^{Z_{i}}|\mathbf{x}_{t}) = 1, 
\end{aligned}
\end{equation}
Therefore we have:
\begin{equation}
\begin{aligned}
    \sum_{s,i} \Omega(\mathbf{x}_{s}^{Z_{i}}|\mathbf{x}_{t}) log \frac{exp(-E_{uni}(\mathbf{x}_{0}, \mathbf{x}_{s}^{Z_{i}},\mathbf{x}_{t}))}{F(\mathbf{x}_{t})}p_{\theta}(\mathbf{x}_{s}^{Z_{i}}|\mathbf{x}_{t}) > \sum_{s,i} \Omega(\mathbf{x}_{s}^{Z_{i}}|\mathbf{x}_{t}) log p_{\theta}(\mathbf{x}_{s}^{Z_{i}}|\mathbf{x}_{t})
\end{aligned}
\end{equation}
Therefore, we have:
\begin{equation}
\begin{aligned}
    KL(\pi||p_{\phi,\theta}) < KL(\pi||p_{\theta}).
\end{aligned}
\end{equation}
\end{proof}

\subsection{Settings on Diffusion Language Model Experiments}
\label{app:dlm_settings}
\subsubsection{Metrics Details}
Perplexity (PPL) is used to evaluate the quality of a language model by measuring how well it predicts a sequence of tokens. Given a sequence $\mathbf{x} = (\mathbf{x}_{1}, \mathbf{x}_{2}, ..., \mathbf{x}_{L})$ of length $L$, PPL is calculated as: $PPL(\mathbf{x}) = exp(-\frac{1}{L} \sum_{i=1}^{L} log p_{\theta}(\mathbf{x}_{i}))$. 

Generative Perplexity (Gen PPL) evaluates the quality of text generated by a model using a reference language model. Instead of conditioning on ground-truth prefixes, it measures how likely a generated sequence is. Given a generated sequence $\hat{\mathbf{x}} = (\hat{\mathbf{x}}_{1}, \hat{\mathbf{x}}_{2}, ..., \hat{\mathbf{x}}_{L})$, Gen PPL is calculated as: $GenPPL(\hat{\mathbf{x}})=exp(-\frac{1}{L} \sum_{i=1}^{L} log p_{eval}(\hat{\mathbf{x}}_{i}|\hat{\mathbf{x}}_{<i}))$, where $p_{eval}$ is the reference model and usually it's a Large Language Model.

\subsubsection{Baseline Settings}
As for auto-regressive Transformer (AR)~\cite{vaswani2017attention}, SEDD~\cite{lou2023discrete}, MDLM~\cite{sahoo2024simple} and EDLM~\cite{xu2024energy}, we use the official model and dataset settings. Therefore, we directly report Perplexity (PPL) from~\cite{xu2024energy}. For Generative Perplexity (Gen PPL), we reproduce results with 128 timesteps, while the remaining results are taken from~\cite{xu2024energy}.

Duo~\cite{sahoo2025diffusion} uses the same dataset settings as ours, and we report its official PPL results. As for Generative Perplexity (Gen PPL), we vary the decoding steps (128, 256, 512, and 1024) and use the official checkpoint, which is trained on the same dataset (OpenWebText).

For Block Diffusion~\cite{arriola2025block}, we set the block size as 4, following the default setting in the original paper, which generally achieves better performance than larger block sizes (e.g., 8 or 16). For Gen PPL, we use the published checkpoint trained on OpenWebText, and evaluate with decoding timesteps of 128, 256, 512, and 1024 while keeping the block size fixed at 4.

For DCD~\cite{liu2024discrete}, we use GPT-2-small~\cite{radford2019language} as the AR proxy (copula) model and MDLM~\cite{sahoo2024simple} as the diffusion backbone.

For ReMDM~\cite{wang2025remasking}, we use MDLM as backbone and set the ReMDM-cap parameter as 0.05 and ReMDM-rescale as 0.5 because it generally leads to better performance in its official report.

\subsection{Comparison with Speculative Decodings}
\label{app:uni_e_comparison_sec}
When integrating with Diffusion Large Language Models (DLLMs), our Uni-E is similar to speculative decoding methods~\cite{israel2025accelerating,gao2025self}, which also use a proxy model to calculate scores or logits to accelerate decoding or improve generation quality. However, the key difference is that \textit{existing speculative decoding methods for DLLMs are generally heuristic and integrate the proxy model without guarantees for invariant decoding}. 

Taking APD~\cite{israel2025accelerating} as an example, it combines the AR proxy model with the DLLM in a linear manner, where the model produces logits as: $model\_logits = Softmax(R*p_{\theta}(\mathbf{x}_{t}|\mathbf{x}_{<t})+ (1-R)*p_{AR}(\mathbf{x}_{t}|\mathbf{x}_{<t}))$, where $p_{\theta}(.)$ and $p_{AR}(.)$ are the DLLM and AR proxy model, respectively. This design can model dependencies to some extend, but cannot capture invariance (i.e., future decoding effects). 

In contrast, is explicitly designed to model both dependency and invariance. Moreover, the formulation of our energy function has a clear theoretical foundation, showing that it can bridge the gap between the ground-truth distribution and the diffusion model distribution (see Eq.~\ref{eq:method_reweight} and Eq.~\ref{eq:id_energy}). A comparison between Uni-E and speculative decoding is summarized in Table~\ref{tab:comparison_uni_e}.
\begin{table}[htbp]
\centering
\small
\caption{Comparison between Speculative Decoding and Uni-E.}
\begin{tabular}{l|cccc}
\toprule
Method & AR Proxy & Dependency & Invariance & Theoretical Guarantee \\
\midrule
Speculative Decoding & Yes & Yes & No & No \\
Uni-E & Yes & Yes & Yes & Yes \\
\bottomrule
\end{tabular}
\label{tab:comparison_uni_e}
\end{table}

\subsection{Comparison with Remask Strategies}
\label{app:uni_e_comparison_remask_sec}
To address token relationship issues, several methods~\cite{wang2025remasking,yao2026remask,wang2025remasking} have proposed using a remask strategy to correct decoding errors. The basic idea behind these methods is to remask previously unmasked tokens during future decoding steps, allowing errors caused by dependency or invariance issues to be corrected. However, these approaches are generally heuristic and lack a theoretical foundation to guarantee dependency or invariance modeling. The selection of tokens to be corrected is also highly dependent on heuristic strategies or manually designed hyperparameters, without theoretical guarantee. For example, ReMDM~\cite{wang2025remasking} uses a hyper-parameter and a scale-score to remask tokens at each step, which view the tokens equally important. Core~\cite{wang2025remasking} remask based on neighbor stability test, which only consider the nearby tokens. None of these methods have theoretical guarantee that the remasking can truly fix the dependency and invariance error. In addition, these methods usually require more decoding timesteps to achieve good performance, since remasking and re-predicting tokens introduces additional decoding steps. A detailed comparison between these methods and our Uni-E is summarized in Table~\ref{tab:comparison_uni_e_remask}.

\begin{table}[htbp]
\centering
\small
\caption{Comparison between Remask Strategy and Uni-E.}
\begin{tabular}{l|cccccc}
\toprule
Method & AR Proxy & Dependency & Invariance & Theoretical Guarantee(Stable) & Efficiency \\
\midrule
Remask & No & Yes & Yes & No & No \\
Uni-E & Yes & Yes & Yes & Yes & Yes \\
\bottomrule
\end{tabular}
\label{tab:comparison_uni_e_remask}
\end{table}

\subsection{Settings of Baselines on Diffusion Large Language Models (DLLMs) Experiments}
\label{app:dllm_exp_settings}
For all the AR models (Qwen2.5-0.5B, Qwen2.5-7B~\cite{yang2025qwen3}, Llama3.1-8B~\cite{grattafiori2024llama}) we use the official model parameters and evaluate them within the TraceRL framework~\cite{wang2025revolutionizing}. We adopt a static decoding strategy and, similar to Uni-E, decode 4 tokens at each step.

For APD~\cite{israel2025accelerating}, we follow the original paper and use Qwen2.5-0.5B as the proxy model, decoding 4 tokens per step.

For DAWN~\cite{luo2026dawn}, we set the confidence thresholds for graph construction as follows: the high-confidence threshold is 0.9, the low-confidence threshold is 0.7, and the anchor-based threshold is 0.7.

For CORE~\cite{zhai2026core}, we use their official implementation and adapt the decoding number as 4 at each step (including the decoding for the remasked tokens).

\subsection{Hyper-parameter Analysis on Diffusion Large Language Model}
\label{app:dllm_hyperparameter_study}

\subsubsection{Candidate Size Study}
\begin{wraptable}{r}{0.5\textwidth}
\centering
\caption{Influence of candidate size $k$.}
\begin{tabular}{c|cc}
\toprule
    & Accuracy & Throughoutput \\
\midrule
k=2 & 38.6    & 40.1          \\
k=3 & 39.2    & 38.8          \\
k=4 & 40.1    & 35.5          \\
k=5 & 39.8    & 31.7          \\
\bottomrule
\end{tabular}
\label{tab:dllm_hyperparameter_k}
\end{wraptable}
We study the impact of the candidate size $k$ on model performance. In this experiment, we use MATH500 as the evaluation benchmark and adopt LLaDa-8B as the base diffusion large language model (DLLM). The results are summarized in Table~\ref{tab:dllm_hyperparameter_k}. We observe that varying the candidate size has only a marginal effect on performance. However, increasing $k$ leads to a noticeable reduction in decoding speed.

\subsubsection{Energy-based Decoding Ratio Study}
\begin{wraptable}{r}{0.5\textwidth}
\centering
\caption{Influence of energy sampling ratio $w$.}
\begin{tabular}{c|cc}
\toprule
    & Accuracy & Throughoutput \\
\midrule
w=0.1 & 37.0    & 43.9          \\
w=0.2 & 39.2    & 38.8          \\
w=0.4 & 41.5    & 27.4          \\
w=0.8 & 42.3    & 11.8          \\
\bottomrule
\end{tabular}
\label{tab:dllm_hyperparameter_w}
\end{wraptable}
By default, we apply energy-based sampling to the first 20\% of the inference steps. To study the effect of this ratio, we vary the energy sampling ratio $w$ and evaluate performance when applying energy sampling to the first 10\%, 20\%, 40\%, and 80\% of the steps. We conduct the experiments on MATH500 using LLaDa-8B as the base model. The results are presented in Table~\ref{tab:dllm_hyperparameter_w}. We find that increasing the sampling ratio generally improves performance. However, this comes at the cost of significantly reduced decoding efficiency. Therefore, we select $w = 0.2$ as a balanced setting between performance and speed.


\subsection{Benchmark Performance of Diffusion Large Language Models (DLLMs)}
\label{app:dllm_bench_full}
We provide the full results for DLLMs in Table~\ref{app_tab:unie_results}. 
\begin{table*}[h!]
\vspace{-3pt}
\centering
\footnotesize
\caption{Several benchmark results. Improvements of applying Uni-Energy to their backbone models are shown in green. $^{\dagger}$ taken from~\cite{wang2025revolutionizing}. $^{\ddagger}$ taken from ~\cite{yang2025qwen3}. $^{*}$ denotes reproduced results.}
\vspace{-5pt}
\setlength{\tabcolsep}{4pt}
\begin{tabular}{l|cc|ccc|cc}
\toprule
Model 
& MATH500 & GSM8K 
& MBPP & LCB V2 & LiveBench 
& MMLU & HellaSwag \\
\midrule

Llama3.1-8B$^{\dagger}$ 
& 51.9 & 84.5 & 42.4 & 20.0 & 19.7 & 66.6 & 76.7 \\

Qwen2.5-7B$^{\dagger}$ 
& 74.0 & 89.9 & 74.9 & 26.9 & 31.1 & 74.2 & 80.2 \\

Qwen2.5-0.5B$^{\ddagger}$
& 33.8 & 41.6 & 39.3 & 10.4 & 12.1 & 47.5 & 52.1 \\

\midrule
LLaDa-8B$^{*}$ 
& 35.2 & 78.1 & 34.4 & 5.8 & 4.1 & 57.3 & 68.2 \\

+DAWN$^{*}$
& 38.3 (\textcolor{darkgreen}{+3.1}) 
& 79.0 (\textcolor{darkgreen}{+0.9}) 
& 33.6 (\textcolor{darkred}{-0.8}) 
& 5.5 (\textcolor{darkred}{-0.3}) 
& 4.8 (\textcolor{darkgreen}{+0.7}) 
& 58.6 (\textcolor{darkgreen}{+1.3}) 
& 68.8 (\textcolor{darkgreen}{+0.6}) \\

+APD$^{*}$
& 36.8 (\textcolor{darkgreen}{+1.6}) 
& 77.5 (\textcolor{darkred}{-0.6}) 
& 32.6 (\textcolor{darkred}{-1.8}) 
& 7.1 (\textcolor{darkgreen}{+1.3}) 
& 6.3 (\textcolor{darkgreen}{+2.2}) 
& 60.5 (\textcolor{darkgreen}{+3.2}) 
& 69.6 (\textcolor{darkgreen}{+1.4}) \\

+CORE$^{*}$
& 35.6 (\textcolor{darkgreen}{+0.4}) 
& 78.7 (\textcolor{darkgreen}{+0.6}) 
& 33.9 (\textcolor{darkred}{-0.5}) 
& 4.9 (\textcolor{darkred}{-0.9}) 
& 3.0 (\textcolor{darkred}{-1.1}) 
& 56.5 (\textcolor{darkred}{-0.8}) 
& 68.0 (\textcolor{darkred}{-0.2}) \\

\rowcolor{gray!15}
+Ind-Energy
& 37.6 (\textcolor{darkgreen}{+2.4}) 
& 79.5 (\textcolor{darkgreen}{+1.4}) 
& 34.9 (\textcolor{darkgreen}{+0.5}) 
& 9.2 (\textcolor{darkgreen}{+3.4}) 
& 8.6 (\textcolor{darkgreen}{+4.5}) 
& 62.7 (\textcolor{darkgreen}{+5.4}) 
& 69.9 (\textcolor{darkgreen}{+1.7}) \\

\rowcolor{gray!15}
+Inv-Energy
& 36.5 (\textcolor{darkgreen}{+1.3}) 
& 78.7 (\textcolor{darkgreen}{+0.6}) 
& 35.0 (\textcolor{darkgreen}{+0.6}) 
& 7.8 (\textcolor{darkgreen}{+2.0}) 
& 5.7 (\textcolor{darkgreen}{+1.6}) 
& 59.4 (\textcolor{darkgreen}{+2.1}) 
& 69.0 (\textcolor{darkgreen}{+0.8}) \\

\rowcolor{gray!15}
+Uni-Energy
& 39.2 (\textcolor{darkgreen}{+4.0}) 
& 80.1 (\textcolor{darkgreen}{+2.0}) 
& 35.2 (\textcolor{darkgreen}{+0.8}) 
& 10.5 (\textcolor{darkgreen}{+4.7}) 
& 10.9 (\textcolor{darkgreen}{+6.8}) 
& 66.5 (\textcolor{darkgreen}{+9.2}) 
& 71.2 (\textcolor{darkgreen}{+3.0}) \\

\midrule

Dream-7B$^{*}$ 
& 34.4 & 69.7 & 55.4 & 7.5 & 7.1 & 61.8 & 70.6 \\

+DAWN$^{*}$
& 35.8 (\textcolor{darkgreen}{+1.4}) 
& 70.5 (\textcolor{darkgreen}{+0.8}) 
& 55.6 (\textcolor{darkgreen}{+0.2}) 
& 7.1 (\textcolor{darkred}{-0.4}) 
& 7.3 (\textcolor{darkgreen}{+0.2}) 
& 61.1 (\textcolor{darkred}{-0.7}) 
& 71.9 (\textcolor{darkgreen}{+1.3}) \\

+APD$^{*}$
& 35.9 (\textcolor{darkgreen}{+1.5}) 
& 71.4 (\textcolor{darkgreen}{+1.7}) 
& 55.8 (\textcolor{darkgreen}{+0.4}) 
& 9.0 (\textcolor{darkgreen}{+1.5}) 
& 8.2 (\textcolor{darkgreen}{+1.1}) 
& 61.3 (\textcolor{darkred}{-0.5}) 
& 71.3 (\textcolor{darkgreen}{+0.7}) \\

+CORE$^{*}$
& 34.6 (\textcolor{darkgreen}{+0.2}) 
& 70.1 (\textcolor{darkgreen}{+0.4}) 
& 56.0 (\textcolor{darkgreen}{+0.6}) 
& 6.3 (\textcolor{darkred}{-1.2}) 
& 6.5 (\textcolor{darkred}{-0.6}) 
& 62.2 (\textcolor{darkgreen}{+0.4}) 
& 69.8 (\textcolor{darkred}{-0.8}) \\


\rowcolor{gray!15}
+Ind-Energy
& 35.7 (\textcolor{darkgreen}{+1.3}) 
& 71.3 (\textcolor{darkgreen}{+1.6}) 
& 56.7 (\textcolor{darkgreen}{+1.3}) 
& 9.8 (\textcolor{darkgreen}{+2.3}) 
& 9.3 (\textcolor{darkgreen}{+2.2}) 
& 62.4 (\textcolor{darkgreen}{+0.6}) 
& 72.6 (\textcolor{darkgreen}{+2.0}) \\

\rowcolor{gray!15}
+Inv-Energy
& 35.2 (\textcolor{darkgreen}{+0.8}) 
& 70.2 (\textcolor{darkgreen}{+0.5}) 
& 56.3 (\textcolor{darkgreen}{+0.9}) 
& 8.1 (\textcolor{darkgreen}{+0.6}) 
& 8.8 (\textcolor{darkgreen}{+1.7}) 
& 62.7 (\textcolor{darkgreen}{+0.9}) 
& 71.4 (\textcolor{darkgreen}{+0.8}) \\

\rowcolor{gray!15}
+Uni-Energy
& 36.4 (\textcolor{darkgreen}{+2.0}) 
& 71.9 (\textcolor{darkgreen}{+2.2}) 
& 57.2 (\textcolor{darkgreen}{+1.8}) 
& 11.1 (\textcolor{darkgreen}{+3.6}) 
& 9.9 (\textcolor{darkgreen}{+2.8}) 
& 63.2 (\textcolor{darkgreen}{+1.4}) 
& 73.7 (\textcolor{darkgreen}{+3.1}) \\

\bottomrule
\end{tabular}
\label{app_tab:unie_results}
\vspace{-2pt}
\end{table*}

\subsection{Case Study}
\label{app:case_study}
To visualize the influence of invariance and independency to text generation, we conduct a case study. We use Uni-EDLM-AR and MDLM for this study. We find that Uni-E generally capture the invariance and the independency correctly and we provide two cases to demonstrate. Besides, we also provides two cases where DLM fails to model the invariance and the independency thus leads to strange generation.

\noindent \textbf{Invariance Case}: The invariant decoding case of our Uni-E is shown in Figure~\ref{app_fig:invariant_case}. In this case, the "sad" can not be fully decided without considering the later context: People "don't have opinions" and city is "not going to pay", thus the decision can not be carried out, then we can deduce the "sad". Uni-E can totally capture the invariance of these tokens and generate the reasonable text.

\begin{figure}[h!]
\centering
\includegraphics[width=0.9\textwidth]{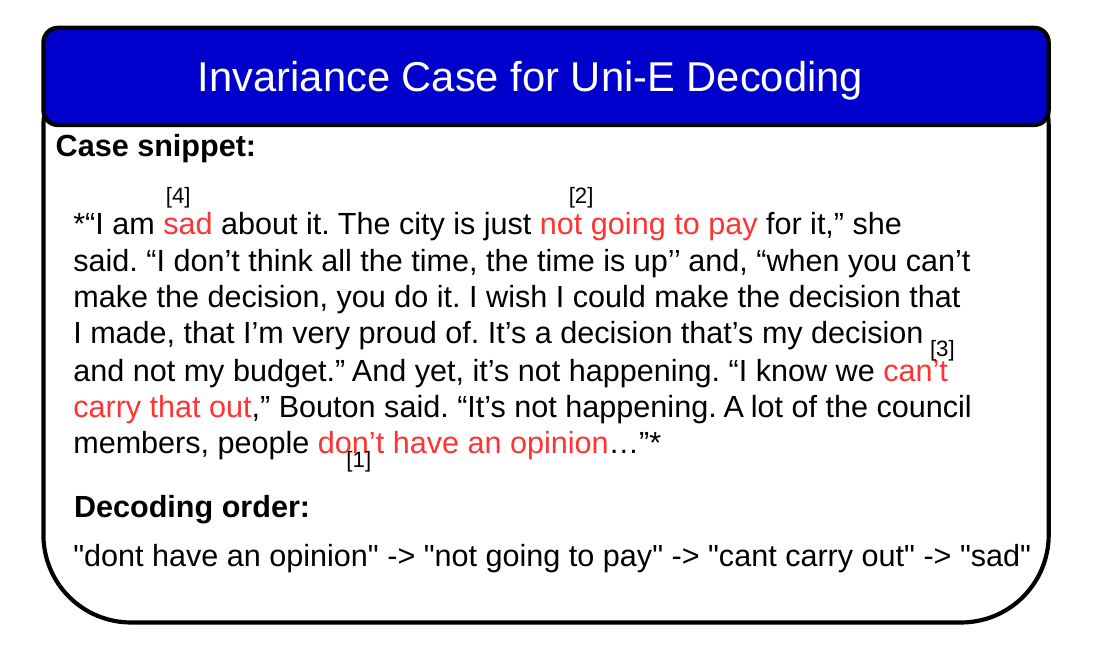}
\caption{The invariant decoding case of Uni-E.}
\label{app_fig:invariant_case}
\end{figure}

\noindent \textbf{Independency Case}: The independent decoding case of our Uni-E is shown in Figure~\ref{app_fig:independent_case}. In this case, although the "not allow" and the "comply with" are adjacent, these tokens are non-independent (If "allow it", then it will "break against" the rule; If "not allow it", then it will "comply with the rule"). Uni-E can capture the dependency between these tokens and correctly model the relationship that "not allow it" leads to "comply with".

\begin{figure}[h!]
\centering
\includegraphics[width=0.9\textwidth]{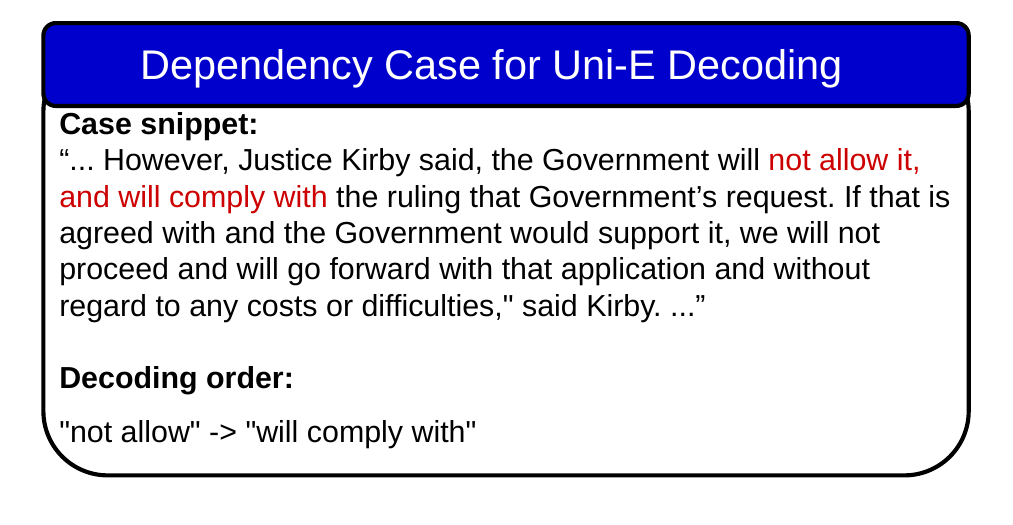}
\caption{The independent decoding case of Uni-E.}
\label{app_fig:independent_case}
\end{figure}

\noindent \textbf{Non-invariance Case}: We give one case where MDLM fails to capture the invariance. The case is shown in Figure~\ref{app_fig:non_invariant_case}. We observe that the text generated by MDLM is weird since there is conflict between "delayed the meeting" and "had a conversation", and we check the decoding order and find that the "delayed the meeting" are decoded before the "had a conversation when ... met with ..." are decoded, which breaks the invariance that the "delayed the meeting" can not be decided before considering the later context. This case shows the limitation of DLM in capturing the invariance.

\begin{figure}[h!]
\centering
\includegraphics[width=0.9\textwidth]{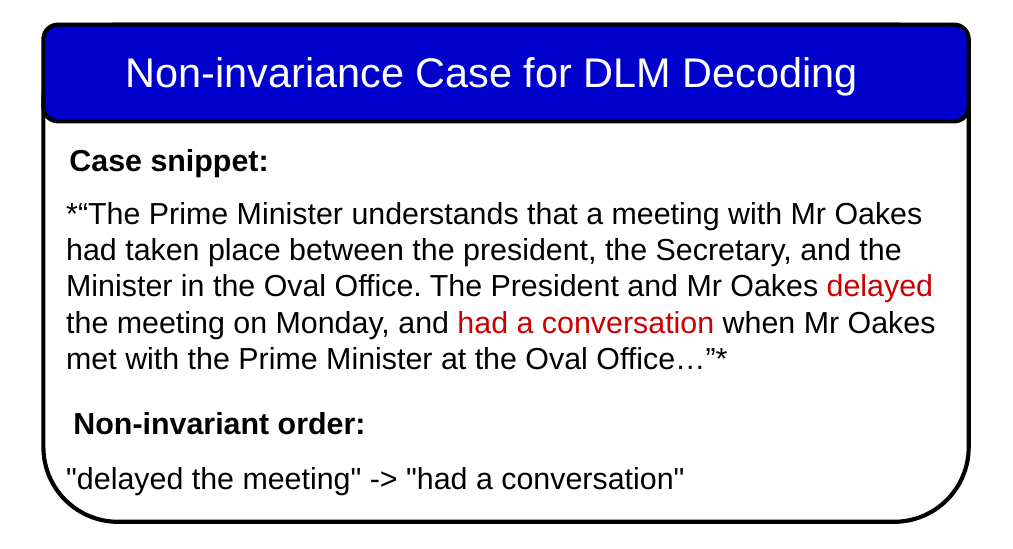}
\caption{The non-invariant decoding case of MDLM.}
\label{app_fig:non_invariant_case}
\end{figure}

\noindent \textbf{Non-independent Case}: We give one case where MDLM fails to capture the dependency. The case is shown in Figure~\ref{app_fig:non_independent_case}. The generated text is unreasonable because "keep the wealth" conflicts with "increase the money lose" and also conflicts with the word "fails". We check the decoding order which shows that the "keep", the "increase" and the "lose" are decoded simultaneously but they are non-independent with each other. Although the model succeeds modeling the invariance between the "fails" and the "increase the money lose", the text is still weird since the "fail" conflicts with "keep the wealth". This case shows the limitation of DLM in capturing the dependency.

\begin{figure}[h!]
\centering
\includegraphics[width=0.9\textwidth]{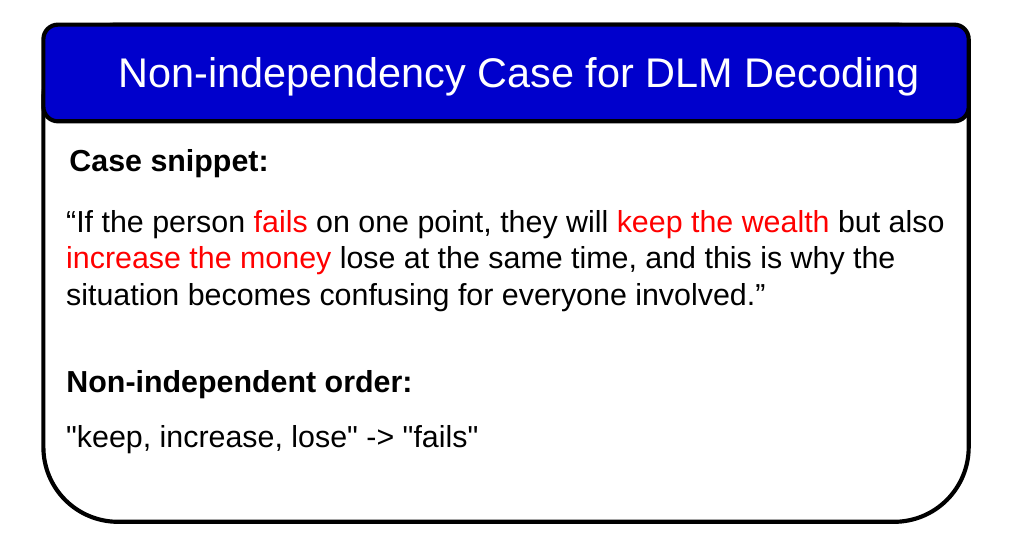}
\caption{The non-independent decoding case of MDLM.}
\label{app_fig:non_independent_case}
\end{figure}

\clearpage

\end{document}